%% file: LocalRankingArxiv.tex
\documentclass{amsart}  
\usepackage[english]{babel}
\usepackage[utf8]{inputenc}
\usepackage{adjustbox}
\usepackage{algorithmic}
\usepackage{algorithm}
\usepackage{amsfonts}
\usepackage{amsmath}
\usepackage{amssymb}
\usepackage{amsthm}
\usepackage{enumerate}
\usepackage{epsfig}
\usepackage{graphicx}
\usepackage{amsaddr}
\usepackage{subfigure}
\usepackage{url}
\usepackage{multirow}

\input{CsdMacros}	
\newtheorem{definition}{Definition}[section]
\newtheorem{theorem}{Theorem}[section]
\newtheorem{corollary}{Corollary}[section]
\newtheorem{lemma}{Lemma}[section]

\title{AUC Optimisation and Collaborative Filtering}

\author{Charanpal Dhanjal \and St\'ephan Cl\'{e}men\c{c}on}
\address[Charanpal Dhanjal]{Institut Mines-T\'{e}l\'{e}com; T\'{e}l\'{e}com ParisTech, CNRS LTCI, 46 rue Barrault, 75634 Paris Cedex 13, France}
\email[Charanpal Dhanjal]{\{charanpal.dhanjal, stephan.clemencon\}@telecom-paristech.fr}
\author{Romaric Gaudel}
\address[Romaric Gaudel]{Université Lille 3, Domaine Universitaire du Pont de Bois, 59653 Villeneuve d'Ascq Cedex, France}
\email[Romaric Gaudel]{romaric.gaudel@univ-lille3.fr}
\date{\today}

\begin{document} 

\begin{abstract} % XXX references forbidden in the abstract
In recommendation systems, one is interested in the ranking of the predicted items as opposed to other losses such as the mean squared error. Although a variety of ways to evaluate rankings exist in the literature, here we focus on the Area Under the ROC Curve (AUC) as it widely used and has a strong theoretical underpinning. In practical recommendation, only items at the top of the ranked list are presented to the users. With this in mind, we propose a class of objective functions over matrix factorisations which primarily represent a smooth surrogate for the real AUC,  and in a special case we show how to prioritise the top of the list. The objectives are differentiable and optimised through a carefully designed stochastic gradient-descent-based algorithm which scales linearly with the size of the data. In the special case of square loss we show how to improve computational complexity by leveraging previously computed measures.  To understand theoretically the underlying matrix factorisation approaches we study both the consistency of the loss functions with respect to AUC, and generalisation using Rademacher theory. The resulting generalisation analysis gives strong motivation for the optimisation under study. Finally, we provide computation results as to the efficacy of the proposed method using synthetic and real data. 
%\keywords{AUC \and collaborative filtering \and implicit recommendation \and matrix factorisation \and rademacher theory}
\end{abstract}

\maketitle

\section{Introduction}

A recommendation system \cite{adomavicius2005toward,koren2009matrix} takes a set of items that a user has rated and recommends new items that the user may like in the future. Such systems have a broad range of applications such as recommending books \cite{linden2003amazon}, CDs \cite{linden2003amazon}, movies \cite{szomszor2007folksonomies,basu1998recommendation} and news \cite{das2007google}. To formalise the recommendation problem let $\{w_1, \ldots, w_m\} \subseteq \mathcal{W}$ be the set of all users and let $\{y_1, \ldots, y_n\} \subseteq \mathcal{Y}$ be the items that can be recommended. Each user $w_i$ rates item $y_j$ with a value $r_{ij}$ which measures whether item $y_j$ is useful to $w_i$. In this work, we consider the \emph{implicit recommendation problem} in which $r_{ij} \in \mathcal{R} = \{-1, +1\}$. For example, the rating value represents whether a person has read a particular research paper, or bought a product. We are thus presented with a set of observations $\{(w_i, y_j, r_{ij}) : w_i \in \mathcal{W},  y_j \in \mathcal{Y}\}$ as a training sample.  The aim of recommendation systems is to find a scoring function $f:  \mathcal{W} \times \mathcal{Y} \mapsto \mathbb{R}$ such that the score $f(w, y)$ is high when user $w$ strongly prefers item $y$. This problem is challenging for a number of reasons: we are interested only in the top few items for each user, there is often a large fraction of missing observations (irrelevant items are generally unknown) and the sample of rated items is often drawn from a non-uniform distribution.  

To address these issues, we propose a general framework for obtaining strong ranking of items based on the theoretically well studied local \emph{Area Under the ROC Curve} (AUC) metric. The framework relies on matrix factorisation to represent parameters, which generally performs well and scales to large numbers of users and items. One essentially finds a low rank approximation of the unobserved entries according to an AUC-based objective and several different surrogate loss functions. In addition we show how to focus the optimisation to the items placed at the top of the list. The resulting methods have smooth differentiable objective functions and can be solved using stochastic gradient descent. We additionally show how to perform the optimisation in parallel so it can make effective use of modern multicore CPUs. The novel algorithms are studied empirically in relation to other state-of-the-art matrix factorisation methods on a selection of synthetic and real datasets. We also answer some theoretical questions about the proposed methods. The first question is whether optimising the surrogate functions will result in improving the AUC. The second question represents a generalisation analysis of the matrix factorisation approaches using \emph{Rademacher Theory} \cite{bartlett03rademacher}, a data-dependent approach to obtaining error bounds. The analysis sheds some light onto whether the quantities optimised on a training sample will generalise to unseen observations and provides a bound between these values. Note that a preliminary version of this paper has been presented in \cite{dhanjal2015collaborative} and here we extend the work by considering a much larger class of objectives, the theoretical study and further empirical analysis. 

This paper is organised as follows. In the next section we motivate and derive the Matrix Factorisation with AUC (MFAUC) framework and present its specialisation according to a variety of objective functions. In Section \ref{sec:theory} we present the theoretical study on consistency and generalisation of the proposed approaches. Following, there is a review on related work on matrix factorisation methods including those based on top-$\ell$ rank-based losses. The MFAUC algorithm is then evaluated empirically in Section \ref{sec:exp} and finally we conclude with a summary and some perspectives.  

\textbf{Notation}: A bold uppercase letter represents a matrix, e.g. $\Xm$, and a column vector is displayed using a bold lowercase letter, e.g. $\xv$. The transpose of a matrix or vector is written $\Xm^T$. The indicator function is given by $I(\cdot)$ so that it is 1 when its input is true otherwise it is . The all ones vector is denoted by $\jv$ and the corresponding matrix is $\Jm$. 

\section{Maximising AUC}\label{sec:MAUC}

The Area Under the ROC Curve is a common way to evaluate rankings. Consider user $w \in \mathcal{W}$ then the AUC is the expectation that a uniformly drawn positive item $y$ is greater with respect to a negative item $y'$ under a scoring function $s: \mathcal{Y} \mapsto \mathbb{R}$
\begin{equation}\label{eqn:auc1}
\mbox{AUC}_\mathcal{D}(s) = \mathbb{E}_{\mathcal{D}}[\mathcal{I}(s(y) > s(y'))], 
\end{equation}
where $\mathcal{D}$ is a distribution over items for $w$, $\mathcal{I}$ is the \emph{indicator function} that is 1 when its input is true and otherwise 0, and we assume scores are never equal. One cannot in general maximise AUC directly since the indicator function is non-smooth and non-differentiable and hence difficult to optimise. Our observations consist only of positive relevance for items, thus we maximise a quantity closely related to the AUC for all users, which is practical to optimise. The missing observations are considered as negative as in \cite{cremonesi2010performance,steck2010training} for example. 

Accuracy experienced by users is closely related to performance on complete data rather than available data \cite{steck2010training} and thus the distribution $\mathcal{D}$ is an important consideration in a practical recommendation system. This implies that a non-uniform sampling of items might be beneficial. Consider a user $w$ and a rating function $s$, then the empirical AUC for this user is: 

\begin{equation} \label{eqn:auc}
\widehat{\mbox{AUC}}(s) = \sum_{p \in \omega}\sum_{q \in \bar{\omega}} \mathcal{I}(s(y_p) > s(y_q))g(y_p) g'(y_q),
\end{equation}
where $\omega$ is the set of indices of relevant items for user $w$, $\bar{\omega}$ is the complement of $\omega$, and $g$ and $g'$ are distributions on the relevant and irrelevant items respectively. The most intuitive choices for $g$ and $g'$ are the uniform distributions. On real datasets however, item ratings often have a long tail distribution so that a small number of items represent large proportions of the total number of ratings. This opens up the question as to whether the so-called \emph{popularity bias} might be an advantage to recommender systems that predict using the same distribution and we return to this point later. 

The AUC certainly captures a great deal of information about the ranking of items for each user, however as pointed out by previous authors, in practice one is only interested in the top items in that ranking. Two scoring functions $s$ and $s'$ could have identical AUC value and yet $s$ for example scores badly on the top few items but recovers lower down the list. One way of measuring this behaviour is through \emph{local AUC} \cite{clemenccon2007ranking} which is the expectation that a positive item is scored higher than a negative one, and the positive one is in the top $t$th quantile. This is equivalent to saying that we ignore positive items low down in the ranking. 

\subsection{A General Framework}\label{sec:general}

As previously stated, direct maximisation of AUC is not practical and hence we use a selection of surrogate functions to approximate the indicator function. Here a scoring function is found for all users and items, in particular the score for the $i$th user and $j$th item is given by $s(w_i, y_j) = (\Um\Vm^T)_{ij}$ where $\Um \in \mathbb{R}^{m \times k}$ is a matrix of $m$ \emph{user factors} and $\Vm \in \mathbb{R}^{n \times k}$ is a matrix of $n$ \emph{item factors}. We will refer to the $i$th rows of these matrices as $\uv_i$ and $\vv_i$ respectively. Furthermore, let $\Xm \in \mathbb{R}^{m \times n}$ be a matrix of ratings such that $\Xm_{ij} = +1$ if user $w_i$ finds item $y_j$ relevant otherwise $\Xm_{ij} = 0$ if the item is missing or irrelevant. In a sense which is made clear later, $\Xm$ is an approximation of the complete structure of learner scores so that $\Xm \approx \Um\Vm^T$.  The advantage of this matrix factorisation strategy is that the scoring function has $k(m+n)$ parameters and hence scales linearly with both the number of users and items. 

The framework we propose is composed principally of solving the following general optimisation: 
\begin{equation}\label{eqn:obj}
\begin{split}
\min \quad&  \frac{1}{m}\sum_{i=1}^m \sum_{p \in \omega_i} g(y_{p}) \phi \left(  \sum_{q \in \bar{\omega}_i} L(\gamma_{i,p,q}) g'(y_{q})\right) + \frac{\lambda}{2}\left(\frac{1}{m}\|\Um\|^2_F + \frac{1}{n}\|\Vm\|^2_F\right), 
\end{split}
\end{equation}
for a user-defined regularisation parameter $\lambda$, loss function $L$, rank weighting function $\phi$, item difference $\gamma_{i,p,q} = \uv_i^T\vv_p - \uv_i^T\vv_q$ and distributions for relevant and irrelevant items $g$ and $g'$. The relevant items for the $i$th user are given by the set $\omega_i$ and irrelevant/missing items are indexed in $\bar{\omega}_i$. The first sum in the first term averages over users and the second sum computes ranking losses over positive items for the $i$th user.  The second term is a penalisation on the Frobenius norms ($\|\Am\|^2_F = \sum_{ij} \Am_{ij}^2$) of the user and item factor matrices. Concrete item distributions $g$ and $g'$ are discussed later in this section but a simple case is using the uniform distributions $1/|\omega_i|$ and $1/|\bar{\omega}_i|$ respectively. 

The loss function $L$ can be specialised to one of the following options ($\beta > 0$ is a user-defined parameter): 
\begin{displaymath}
\begin{array}{l l}
L(x) = \frac{1}{2}\max(0, 1-x)^2 & \mbox{square hinge}  \\
L(x) = \frac{1}{2}(1-x)^2 & \mbox{square} \\
L(x) = -1/(1+e^{-\beta x}) & \mbox{sigmoid} \\ 
L(x) = -\ln(1/(1+e^{-\beta x})) & \mbox{logistic}
\end{array}
\end{displaymath}
See Figure \ref{fig:lossPlots} for a graphical representation of each of these loss functions. On binary classification the square hinge loss is show to provide a strong AUC on both training and test data in \cite{yan2003optimizing}. Furthermore, the objective becomes convex in $\Um$ and $\Vm$ when $\phi(x) = x$. The squared loss is shown to be consistent with the AUC \cite{gao2013one} and the squared hinge loss is shown to be consistent in \cite{gao2012consistency}. In contrast, the hinge loss is not consistent with AUC. The sigmoid function is perhaps the best approximation of the negative indicator function. As $\beta \rightarrow \infty$ it approaches the indicator function, and hence we get an objective exactly corresponding to maximising AUC. The sigmoid function is used in conjunction with AUC for bipartite ranking in \cite{herschtal2004optimising}. An noted in this paper, if $\beta$ is too small then corresponding objective is a poor approximation of the AUC and if it is too large then the objective approaches the sum of step functions making it hard to perform gradient descent. To alleviate the problem the training data is used to to choose a series of increasing $\beta$ values.  

\begin{figure}[ht]
\begin{center}
\includegraphics[width=.5\linewidth]{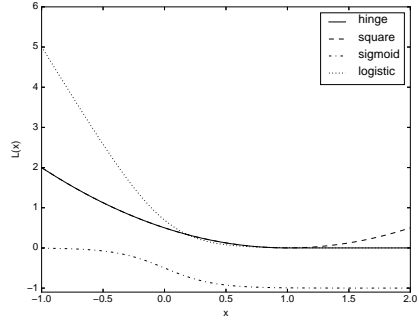} 
\end{center}
\caption{Plot of the loss functions $L$ with $\beta=5$.}
\label{fig:lossPlots}
\end{figure} 

For the weighting function $\phi(x) = x$, one can see that the first term in the objective follows directly from AUC (Equation \ref{eqn:auc}) by replacing the indicator function with the appropriate loss. The optimisation in this case is convex in the square and hinge loss cases for $\Um$ and $\Vm$ but not both simultaneously. This form of the objective however, does not take into account the preference for ranking items correctly at the top of the list. Consider instead $\phi(x) = \tanh(\rho x)$ for $x \geq 0$ and fixed $\rho > 0$, which is a concave function.  The term inside $\phi$ in Optimisation \ref{eqn:obj} represents a ranking loss for $i$th user and $p$th item and thus $y_p$ is high up in the ranked list if this quantity is small. The effect of choosing the hyperbolic tangent is that items with small losses have larger gradients towards the optimal solution and hence are prioritised. 

The distributions on the items allow some flexibility into the expectation we ultimately compute. Although the obvious choice is a uniform distribution, in practice relevant item distributions often follow the so-called \emph{power law} given by $p(y) \propto n_y^{-\tau}$ for some exponent $\tau \geq 1$, where $n_y$ is the number of times relevant item $y$ occurs in the complete (fully observed) data. In words, the power law says that there are a few items which are rated very frequently whereas most items are not frequently rated, corresponding to a bias on observing a rating for popular items. However, recommendations for less well-known items are considered valuable for users. Since we have incomplete data the generating distribution can be modelled in a similar way to that of the complete data (see e.g. \cite{steck2011item} for further motivation), 
\begin{displaymath} 
g(y) \propto \hat{p}(y)^{\tau'}. 
\end{displaymath}
where $\hat{p}(y)$ is the empirical probability of observing $y$ and $\tau' \geq 0$ is a exponent used to control the weight of items in the objective. The irrelevant and missing items form the complement of the relevant ones and hence we have  
\begin{displaymath} 
g'(y) \propto \hat{q}(y)^{\tau'} =  (1 - \hat{p}(y))^{\tau'}. 
\end{displaymath}
Notice that when $\tau' = 0$ we sample from the uniform distribution. Since we expect $\hat{p}(y)$ to be related to $n_y$ with a power law this implies that when $\tau' > 0$ we give preference to items for which $n_y$ is small and hence focus on less popular items. 

\subsection{Optimisation Algorithms}\label{subsec:opt}

%\subsubsection{Hinge and Square loss}

To solve the above objectives one can use gradient descent methods which rely on computing the derivatives with respect to the parameters $\Um$ and $\Vm$. Here we present the derivatives for choices of loss $L$ and weighting function $\phi$, starting with the squared hinge loss and $\phi(x) = x$. Denote the objective function of Optimisation \ref{eqn:obj} as $\theta$ then the derivatives are, 
\begin{equation}\label{eqn:duHinge}
\begin{split}
 \frac{\delta \theta}{\delta \uv_i} &= \frac{1}{m} \sum_{p \in \omega_i} g(y_{p})  \left(\sum_{q \in \bar{\omega}_i}  (\vv_q - \vv_p)h(\gamma_{i,p,q}) g'(y_{q})\right)  + \frac{\lambda}{m} \uv_i,
\end{split}
\end{equation}
and
\begin{equation} \label{eqn:dvHinge}
\begin{split}
\frac{\delta \theta}{\delta \vv_j} =&   \frac{1}{m}  \sum_{i=1}^m  \uv_i \left( \mathcal{I}_{j \in \bar{\omega}_i} g'(y_{j}) \sum_{p \in \omega_i} g(y_{p})  h(\gamma_{i,p,j})  \right. \\ 
& \left.  -  \mathcal{I}_{j \in \omega_i}g(y_{j}) \sum_{q \in \bar{\omega}_i}  g'(y_{q}) h(\gamma_{i,j,q}) \right) + \frac{\lambda}{n} \vv_j,
\end{split}
\end{equation}
where $h(x) = \max(0, 1-x)$ and for convenience we use the notation $\mathcal{I}_{j \in \bar{\omega}_i} = \mathcal{I}(j \in \bar{\omega}_i)$. The squared loss is identical except that $h(x) = (1-x)$ in this case. It is worth noting that the term inside the outer sum of this derivative in the squared loss case can be written as: 
\begin{equation} 
\begin{split}
\frac{\delta \theta}{\delta \vv_j} =&   \frac{1}{m}  \sum_{i=1}^m  \uv_i \left( \mathcal{I}_{j \in \bar{\omega}_i} g'(y_{j})  (1 + \uv_i^T(\vv_j - \dot{\vv}_i)) \right. \\ 
 & \left. - \mathcal{I}_{j \in \omega_i}g(y_{j}) (1 + \uv_i^T( \ddot{\vv}_i - \vv_j)) \right) + \frac{\lambda}{n} \vv_j
\end{split}
\end{equation}
where $\dot{\vv}_i = \sum_{p \in \omega_i} g(y_{p}) \vv_p$ and $\ddot{\vv}_i = \sum_{q \in \bar{\omega}_i} g'(y_{q}) \vv_q$ are empirical expectations. The derivative with respect to $\uv_i$ can be treated in a similar way: 
\begin{equation} 
\frac{\delta \theta}{\delta \uv_i} = \frac{1}{m} (\ddot{\vv}_i - \dot{\vv_i} + \ddot{\wv}_i - \dot{\vv}_i \ddot{\vv}_i^T\uv_i - \ddot{\vv}_i \dot{\vv}_i^T\uv_i + \dot{\wv}_i) + \frac{\lambda}{m} \uv_i, 
\end{equation}
where $\dot{\wv}_i = \sum_{p \in \omega_i} g(y_{p}) \vv_p \vv_p^T\uv_i$ and $\ddot{\wv}_i = \sum_{q \in \bar{\omega}_i} g'(y_{q}) \vv_q \vv_q^T\uv_i$. Thus we have inexpensive ways to compute derivatives in the square loss case provided we have access to the expectations of particular vectors.  

%\subsubsection{Logistic and Sigmoid loss}

The logistic and sigmoid losses are similar functions and their derivatives are computed as follows ($\beta$ is a user-defined parameter and $\phi(x) = x$ as before): 
\begin{equation}
\begin{split}
 \frac{\delta \theta}{\delta \uv_i} &= -\frac{\beta}{m} \sum_{p \in \omega_i} g(y_{p})  \left(\sum_{q \in \bar{\omega}_i}  (\vv_q - \vv_p)h(\gamma_{i,p,q}) g'(y_{q})\right)  + \frac{\lambda}{n} \uv_i,
\end{split}
\end{equation}
and 
\begin{equation} 
\begin{split}
\frac{\delta \theta}{\delta \vv_j} =&   -\frac{\beta}{m}  \sum_{i=1}^m  \uv_i \left( \mathcal{I}_{j \in \bar{\omega}_i} g'(y_{j}) \sum_{p \in \omega_i} g(y_{p})  h(\gamma_{i,p,j}) \right. \\ 
 & \left .- \mathcal{I}_{j \in \omega_i}g(y_{j}) \sum_{q \in \bar{\omega}_i}  g'(y_{q}) h(\gamma_{i,j,q}) \right) + \frac{\lambda}{n} \vv_j,
\end{split}
\end{equation}
in which $h(x) = \frac{e^{-\beta x}}{1 + e^{-\beta x}}$ for the logistic loss and $h(x) = \frac{1}{(1 + e^{-\beta x})^2}$ for the sigmoid loss.

%\subsubsection{Hyperbolic Tangent Weighting}

We now consider the case in which we have the weighting function $\phi(x) = \tanh(\rho x)$ on the item losses in conjunction with a square or squared hinge loss. The gradients with respect to the objective $\theta$ are 
\begin{equation*}
\begin{split}
 \frac{\delta \theta}{\delta \uv_i} &= \frac{\rho}{m} \sum_{p \in \omega_i} g(y_{p})  \left(\sum_{q \in \bar{\omega}_i}  (\vv_q - \vv_p)h(\gamma_{i,p,q}) g'(y_{q})\right) 
 \\ &\left(1- \tanh^2\left( \frac{\rho}{2} \sum_{q \in \bar{\omega}_i} h(\gamma_{i,p,q})^2 g'(y_{q})\right)\right)  + \frac{\lambda}{m} \uv_i,
\end{split}
\end{equation*}
and the corresponding gradient with respect to $\vv_j$ is 
\begin{equation*} 
\begin{split}
\frac{\delta \theta}{\delta \vv_j} &=   \frac{\rho}{m}  \sum_{i=1}^m  \uv_i \left( \mathcal{I}_{j \in \bar{\omega}_i} g'(y_{j}) \sum_{p \in \omega_i} g(y_{p})  h(\gamma_{i,p,j}) \right. \cdot \left(1 - \tanh^2\left( \frac{\rho}{2}\sum_{q \in \bar{\omega}_i} h(\gamma_{i,p,q})^2 g'(y_{q}) \right) \right)    \\    
&  - \mathcal{I}_{j \in \omega_i}g(y_{j}) \sum_{q \in \bar{\omega}_i}  g'(y_{q}) h(\gamma_{i,j,q}) \cdot \left. \left(1 - \tanh^2\left( \frac{\rho}{2} \sum_{\ell \in \bar{\omega}_i} h(\gamma_{i,j,\ell})^2 g'(y_{\ell}) \right) \right)\right) + \frac{\lambda}{n} \vv_j,
\end{split}
\end{equation*}
with $h(x) = \max(0, 1-x)$ in the square hinge loss case and $h(x) = (1-x)$ for the square loss. When we compare the above derivatives to Equations \ref{eqn:duHinge} and \ref{eqn:dvHinge} for the square/hinge loss functions we can see that there is an additional weighting on the relevant items given by $f(\uv_i, \vv_p) = 1- \tanh^2\left( \frac{\rho}{2} \sum_{q \in \bar{\omega}_i} h(\uv_i^T\vv_p - \uv_i^T\vv_q)^2 g'(y_{q})\right)$ which will naturally increase the size of gradient for items with small losses and hence those high up in the ranking.  

Thus we have presented the derivatives for each of the loss functions we proposed earlier in this section as well the hyperbolic tangent weighting scheme. The derivative with respect to $\vv_j$ is generally the most costly to find and we can see that the computational complexity of computing it is $\mathcal{O}(mn)$ and hence the derivative with respect to the complete matrix $\Vm$ is $\mathcal{O}(mn^2)$. The complexity of the derivative with respect to $\Um$ is identical although in practice it is quicker to compute. Fortunately, the expectations for both derivatives can be estimated effectively using $\kappa_\mathcal{W}$ users, and $\kappa_\mathcal{Y}$ items from $\omega_i$ and $\bar{\omega}_i$. This reduces the complexity per iteration of computing the derivative with respect to $\Vm$ to $\mathcal{O}(\kappa_\mathcal{W}\kappa_\mathcal{Y}^2)$. 

Using these approximate derivatives we can apply stochastic gradient descent to solve Optimisation \ref{eqn:obj}. Define $\theta'(\Um, \Vm)$ as the approximate subsampled objective, then one then walks in the negative direction of the approximate derivatives using \emph{average stochastic gradient descent} (average SGD, \cite{polyak1992acceleration}). Algorithm \ref{alg:maxAuc} shows the pseudo code for solving the optimisation. Given initial values of $\Um$ and $\Vm$ (using random Gaussian elements, for example) we define an iteration as $\max(m, n)$ gradient steps using a random permutation of indices $i \in [1, m]$ and $j \in [1, n]$. It follows that each iteration updates all of the rows of $\Um$ and $\Vm$ at least once according to the learning rate $\alpha \in \mathbb{R}^+$ and the corresponding derivative. Note in line \ref{line:assignUv} that if $t$ exceeds the size of the corresponding vector we wrap around to the start. Under average SGD one maintains, during the optimisation, matrices $\bar{\Um}$ and $\bar{\Vm}$ which are weighted averages of the user and item factors. These matrices can be shown to converge more rapidly than $\Um$ and $\Vm$, see \cite{polyak1992acceleration} for further details. The algorithm concludes when the maximum number of iterations $T$ has been reached or convergence is achieved by looking at the average objective value for the the most recent iterations. In the following text we refer to this algorithm as Matrix Factorisation with AUC (MFAUC).

\begin{algorithm}
\begin{algorithmic}[1]
\REQUIRE Ratings $\Xm$, solutions $\Um$, $\Vm$, iterations $T$, average start $T_0$, learning rate $\alpha$, convergence threshold $\epsilon$  
\STATE $\bar{\Um} = \Um$ and $\bar{\Vm} = \Vm$
\STATE Vector $\zv$ is permutation of $\{1, \ldots, m\}$ and $\zv'$ is permutation of $\{1, \ldots, n\}$
\STATE $\theta'_0 = \theta'(\Um, \Vm)$, $\theta'_1 = \theta'_0 + \epsilon$,  $s=1$
\WHILE{$s \neq T$ and $|\theta'_s - \theta'_{s-1}| \geq \epsilon$}
  \FOR{$t = 1,\ldots,\max(m, n)$}
    \STATE Let $i = \zv_{t}$ and $j = \zv'_{t}$ \label{line:assignUv}
    \STATE $\uv_i \leftarrow \uv_i - \alpha \frac{\delta \theta'}{\delta \uv_i}$ and $\vv_j \leftarrow \vv_j - \alpha \frac{\delta \theta'}{\delta \vv_j}$
    \IF{$T \geq T_0$} 
      \STATE $\bar{\uv}_i \leftarrow \frac{i}{i+1}\bar{\uv}_i + \frac{1}{i+1}\uv_i$ and $\bar{\vv}_j \leftarrow \frac{i}{i+1}\bar{\vv}_j  + \frac{1}{i+1}\vv_j$
     \ELSE 
     \STATE $\bar{\uv}_i = \uv_i$ and $\bar{\vv}_j = \vv_j$ 
    \ENDIF 
  \ENDFOR 
  \STATE Update $\theta'_s = \theta'(\bar{\Um}, \bar{\Vm})$, $s = s+1$ 
\ENDWHILE 
\RETURN Solutions $\bar{\Um}$, $\bar{\Vm}$
\end{algorithmic}
\caption{Pseudo code for Matrix Factorisation with AUC}\label{alg:maxAuc}
\end{algorithm}

\subsubsection{Parallel Algorithm}

A disadvantage of the optimisation just described is that it cannot easily be implemented to make use of model parallel computer architectures as each intermediate solution depends on the previous one. To compensate, we build upon the  Distributed SGD (DSDG) algorithm of \cite{gemulla2011large}. The algorithm works on the assumption that the loss function we are optimising can be split up into a sum of local losses, each local loss working on a subset of the elements of $\Xm_{ij}$. The algorithm proceeds by partitioning $\Xm$ into $d_1 \times d_2$ fixed blocks (sub-matrices) and each block is processed by a separate process/thread ensuring that no two concurrent blocks share the same row or column. This constraint is required so that updates to the rows of $\Um$ and $\Vm$ are independent, see Figure \ref{fig:parallelSGD}. The blocks do not have to be contiguous or of uniform size, and in our implementation blocks are sized so that the number of nonzero elements within each one is approximately constant. The algorithm terminates after every block has been processed at least $T$ times. 

\begin{figure}
\begin{center}
\includegraphics[width=.5\linewidth]{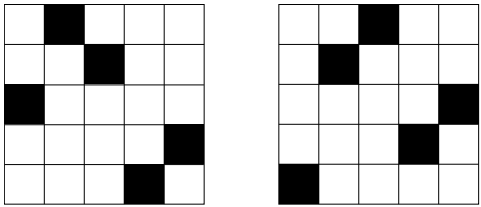} 
\end{center}
\caption{Illustration of the principal behind DSGD. Each of the two large squares represents a matrix divided into blocks, and each black square represents a process working on a block.}
\label{fig:parallelSGD}
\end{figure} 

For AUC-based losses we require a pairwise comparison of items for each row and therefore the loss cannot easily be split into local losses. To get around this issue we modify DSGD as follows: at each iteration we randomly assign rows/columns to blocks, i.e. blocks are no longer fixed throughout the algorithm. We then concurrently compute empirical estimates of the loss above within all blocks by restricting the corresponding positive and negative items, and repeat this process until convergence or for $T$ iterations. 

\section{Consistency and Generalisation} \label{sec:theory}

In this section we will discuss issues relating to the consistency of the above optimisation, as well as the generalisation performance on unseen data. First we address the question of whether optimising a surrogate of the AUC will lead to improving the AUC itself.  We draw upon the work of \cite{gao2012consistency} to show that our chosen loss functions are consistent. In the bipartite ranking case the square hinge loss is shown to be consistent in \cite{gao2012consistency} and a similar result for the square loss is proven in \cite{gao2013one}. Therefore we concentrate on the sigmoid and logistic functions and additionally show how the consistency results apply to the matrix factorisation scenario we consider in this paper. 

To begin we define consistency in a more formal sense, starting with a more general definition of the AUC than Equation \ref{eqn:auc1}. In this case, the scoring function $s$ can result in the same scores for different items,
\begin{displaymath} 
\mbox{AUC}_\mathcal{D}(s) = \mathbb{E}_{\mathcal{D}}[\mathcal{I}((r - r')(s(y) - s(y')) > 0) +  \frac{1}{2}\mathcal{I}(s(y) = s(y')) | r \neq r')], 
\end{displaymath}
where $r, r' \in \{-1, +1\}$ are the labels respectively of items $y, y'$. We can alternatively state the AUC in terms of a risk so that maximising AUC corresponds to minimising this quantity: 
\begin{displaymath} 
R(s) = \mathbb{E}_{(y, r), (y', r') \sim \mathcal{D}}[\ell(s, (y, r), (y', r')) | r \neq r'] 
\end{displaymath}
in which $\ell(s, (y, r), (y', r')) = \mathcal{I}((r - r')(s(y) - s(y')) < 0) +  \frac{1}{2}\mathcal{I}(s(y) = s(y'))$. If we define $\eta(y) = P[r = +1|y]$ then we see that the risk can be expressed as 
\begin{displaymath} 
R(s) \propto \mathbb{E}_{(y, y') \sim \mathcal{D}_\mathcal{Y}^2}[\eta(y)(1-\eta(y'))\ell'(s, y, y') + \eta(y')(1-\eta(y))\ell'(s, y', y)],  
\end{displaymath}
where $\ell'(s, y, y') = \mathcal{I}(s(y) - s(y') < 0) + \frac{1}{2}\mathcal{I}(s(y) = s(y'))$ and $\mathcal{D}_\mathcal{Y}$ is the marginal distribution over $\mathcal{Y}$. Notice that if we assume that $\eta(y) > \eta(y')$ then we would prefer a function $s$ such that $s(y) > s(y')$, and a similar results holds if we swap $y$ and $y'$. This allows us to introduce the \emph{Bayes risk} $R(s^*)$ where $s^*$ is defined as follows: 
\begin{eqnarray*} 
s^* &=& \mbox{arginf}_s R(s)\ \\ 
&=& \{s: (s(y) - s(y'))(\eta(y) - \eta(y')) > 0 \mbox{ if } \eta(y) \neq \eta(y') \},
\end{eqnarray*}
where $s$ is chosen from all measurable functions. Since we replace the indicator with a surrogate loss function, define $L'(s, y, y') = L(s(y) - s(y'))$ then we can write the corresponding risk as, 
\begin{displaymath} 
R_L(s) \propto \mathbb{E}_{(y, y') \sim \mathcal{D}_\mathcal{Y}^2}[\eta(y)(1-\eta(y'))L'(s, y, y') + \eta(y')(1-\eta(y))L'(s, y', y)],   
\end{displaymath}
and the optimal scoring function with respect to this risk is $s^*_L =  \mbox{arginf}_s R_L(s)$. With these definitions, we can define consistency.

\begin{definition}[Consistency,  \cite{gao2012consistency}] 
The surrogate loss $L$ is said to be consistent with AUC if for every sequence $\{s^{\langle i \rangle}(y)\}_{i \geq 1}$, the following holds over all distributions $\mathcal{D}$ on $\mathcal{Y} \times \mathcal{R}$: 
\begin{displaymath} 
 R_L(s^{\langle i \rangle}) \rightarrow R_L(s^*_L) \mbox{ then } R(s^{\langle i \rangle}) \rightarrow R(s^*).
\end{displaymath}
\end{definition}
Thus, a consistent loss function will lead to an optimal Bayes risk in the limit of an infinite sample size. Furthermore, a sufficient condition is given for AUC as follows 
\begin{theorem}[Sufficiency for AUC consistency, \cite{gao2012consistency}]  \label{thm:sufficiency}
The surrogate loss  $L'(s, y, y') = L(s(y) - s(y'))$ is consistent with AUC if $L: \mathbb{R} \rightarrow \mathbb{R}$ is a convex, differentiable and non-increasing function with $\Delta L(0) < 0$. 
\end{theorem}
These theorems allow us to prove that the sigmoid and logistic losses are consistent with AUC. 
\begin{theorem} 
Both the sigmoid $-\sigma(x)$, $\sigma(x) = 1/(1+e^{-\beta x})$, and logistic loss $-\ln(\sigma(x)) = \ln(1/(1+e^{-\beta x}))$ functions are consistent with AUC.  
\end{theorem}
\begin{proof}
We will start with the logistic function. Assume that $\beta > 0$ and $x$ is finite,
\begin{displaymath} 
\frac{\delta -\ln(\sigma)}{\delta x} = -\frac{-\beta e^{-\beta x}}{(1+e^{-\beta x})} < 0,  
\end{displaymath}
which implies a non-increasing function, in particular the derivative at zero is $-\beta/2$. In addition the logistic loss is convex since for all finite $x$,
\begin{displaymath} 
 \frac{\delta^2 -\ln(\sigma)}{\delta x^2} = \frac{\beta^2 e^{-\beta x}}{(1+e^{-\beta x})} \left(1 - \frac{e^{-\beta x}}{(1+e^{-\beta x})} \right) > 0, 
\end{displaymath}
where the first term in the derivative is greater than zero and the second term in parenthesis is between 0 and 1. An application of Theorem  \ref{thm:sufficiency} gives the required result. 

For the sigmoid function, consider a set of items $S = \{y_1, y_2, \ldots, y_n\}$ with corresponding conditional probabilities on being positive $\eta(y_1), \eta(y_2), \ldots, \eta(y_m)$ then we can find the following risk (we use notation $\eta_i = \eta(y_i)$ and $s_i = s(y_i)$ for convenience) 
\begin{displaymath} 
 R_L(s) \propto R'_L(s) = - \sum_{i < j} \left( \eta_i(1- \eta_j) \frac{1}{1+e^{-(s_i - s_j)}} + \eta_j(1- \eta_i) \frac{1}{1+e^{-(s_j - s_i)}} \right). 
\end{displaymath}
For each pair of terms in this sum notice first that $1/(1+e^{-x}) + 1/(1+e^x) = 1$ for all $x$. It follows that if $\eta_i > \eta_j$ then the first term should be maximised by increasing $s_i - s_j$ otherwise the second one should be maximised to minimise the risk. Therefore, in the limiting case 
\begin{displaymath} 
 R'_L(s)  \rightarrow - \sum_{\eta_i > \eta_j} \eta_i(1- \eta_j) ,
\end{displaymath}
 and we have $(s_i - s_j)(\eta_i > \eta_j) > 0$ when $\eta_i \neq \eta_j$ as in the Bayes risk. 
\end{proof}

As we have not made any assumption on scoring functions in our previous results, the AUC is simply generalised to the expectation over users as follows
\begin{displaymath} 
\mbox{AUC}_\mathcal{D}(s) = \mathbb{E}_\mathcal{W \times \mathcal{Y}}[\mathcal{I}((r - r')(s(w, y) - s(w, y')) > 0) +  \frac{1}{2}\mathcal{I}(s(y) = s(y')) | r \neq r')],
\end{displaymath}
where $\mathcal{D}$ is now is a distribution over users and items and we redefine the scoring function as $s: \mathcal{W} \times \mathcal{Y} \mapsto \mathbb{R}$. Thus we have shown that all the loss functions we consider are consistent with the AUC in the multi-user scenario. 

\subsection{Generalisation Bound} 

We now turn to another critical question about our algorithm relating to the generalisation. In particular we look at Rademacher theory, which is a data-dependent way of studying generalisation performance. We assume that the observations are sampled from a distribution $\bar{\mathcal{D}}$ over $\mathcal{W} \times \mathcal{Y} \times \mathcal{R}$. Recall that only positive ratings are observed, hence the sample $S = \cup_{i=1}^m \{(w_i, y_{\omega_{i1}}), \ldots, (w_i, y_{\omega_{in_i}})\}$ is drawn from the distribution $\mathcal{D} = \mathcal{D}_1^{n_1} \times \cdots \times \mathcal{D}_m^{n_m}$ where we use the shorthand $n_i = |\omega_i|$. Now consider a class of functions $\mathcal{Q}$ mapping from $\mathcal{W}\times\mathcal{Y}^2$ to $[0, 1]$ then the AUC can be maximised by minimising 
\begin{equation}\label{eqn:expQ}
\hat{\mathbb{E}}_S[Q] = \frac{1}{m}\sum_{i=1}^m \frac{1}{n_i n_i'} \sum_{p \in \omega_i} \sum_{q \in \bar{\omega_i}} Q(w_i, y_{p}, y_{q}), 
\end{equation}
in the case that $\mathcal{Q} = \{Q: (w, y, y') \mapsto \mathcal{I}(s(w, y) \leq s(w, y')), s \in \mathcal{S}\}$. Likewise we can substitute any class of loss functions $\mathcal{Q}$ to be minimised, for example the surrogate functions defined above. Noting this, we can now present the idea of a Rademacher variables which are uniformly distributed $\{-1, +1\}$ independent random variables. The following definition is an empirical Rademacher average for all users 
\begin{displaymath} 
\hat{R}^{AUC}(\mathcal{Q}) = 2\mathbb{E}_{\nu} \left[\sup_{Q \in \mathcal{Q}} \frac{1}{m}\sum_{i=1}^m \frac{\nu_i}{n_i n_i'} \sum_{p \in \omega_i} \sum_{q \in \bar{\omega_i}} Q(w_i, y_{p}, y_{q}) \right],  
\end{displaymath}
where the $\nu_i$'s are independent Rademacher random variables (note the close connection to Equation \ref{eqn:expQ}). This expectation is an empirical measure of the capacity of a function class because it measures how well functions in that class can correlate with random data. We define the Rademacher complexity of $\mathcal{Q}$ as 
\begin{displaymath} 
R^{AUC}(\mathcal{Q}) = \mathbb{E}_S[\hat{R}^{AUC}(\mathcal{Q})], 
\end{displaymath}
where $S$ is drawn over $\mathcal{D}_1^{n_1} \times \cdots \times \mathcal{D}_m^{n_m}$. In the following lemma we provide a relationship between these two quantities using McDiarmid's Theorem.
\begin{theorem}[McDiarmid, \cite{McDiarmid1989}]\label{thm:mcdiarmid}
Let $X_1, \ldots, X_n$ be independent random variables taking values in a set $A$ and assume $f: A^n \rightarrow \mathbb{R}$ satisfies 
\begin{displaymath} 
\sup_{x_1, \ldots, x_n, \hat{x}_i \in A} |f(x_1, \ldots, x_n) - f(x_1, \ldots, \hat{x_i}, x_{i+1}, \ldots, x_n) | \leq c_i,    
\end{displaymath}
for all $i=1,\ldots,n$. Then for all $\epsilon > 0$, 
\begin{displaymath} 
P\{f(X_1, \ldots, X_n) - \mathbb{E}[f(X_1, \ldots, X_n)] \geq \epsilon \} \leq \exp\left(\frac{-2\epsilon^2}{\sum_{i=1}^n c_i^2}\right), 
\end{displaymath}
and 
\begin{displaymath} 
P\{\mathbb{E}[f(X_1, \ldots, X_n)] - f(X_1, \ldots, X_n)  \geq \epsilon \} \leq \exp\left(\frac{-2\epsilon^2}{\sum_{i=1}^n c_i^2}\right). 
\end{displaymath}
\end{theorem}

We show that the Rademacher complexities are related as follows
\begin{theorem}\label{thm:rad} 
Let $\mathcal{Q}$ be a function class mapping from  $\mathcal{W}\times\mathcal{Y}^2$ to $[0, 1]$ and define sample $S = \cup_{i=1}^m \{(w_i, y_{\omega_{i1}}), \ldots, (w_i, y_{\omega_{in_i}})\}$ drawn from distribution $\mathcal{D}_1^{n_1} \times \cdots \times \mathcal{D}_m^{n_m}$. Then with probability at least $1-\delta$ the following is true 
\begin{displaymath} 
 R^{AUC}(\mathcal{Q}) \leq \hat{R}^{AUC}(\mathcal{Q})  +  \sqrt{\frac{2\ln(1/\delta) (n-1)^2}{m^2} \sum_{i=1}^{m} \frac{1}{|\omega_i||\bar{\omega}_i|^2}}.
\end{displaymath}
\end{theorem}
\begin{proof}
 As well as $S$ consider another sample $\hat{S}_{ab} = S \setminus \{(w_{a}, y_{b})\} \cup  \{(w_{a}, y_{b'})\}$ for $a \in [1,m]$ and $b, b' \in [1, n]$. For notational convenience we call the nonzero elements of $\hat{S}_{ab}$, $\omega_i' = \omega_i$, $i \neq a$ and $\omega_a' = \omega_i \setminus \{y_b\} \cup \{y_{b'}\}$. Likewise for the zero elements $\bar{\omega}_i' = \bar{\omega}_i$, $i \neq a$ and $\bar{\omega}_a' = \bar{\omega}_i \setminus \{y_{b'}\} \cup \{y_{b}\}$. Define 
 \begin{displaymath} 
 f(S) =  2\mathbb{E}_{\nu} \left[\sup_{Q \in \mathcal{Q}} \frac{1}{m}\sum_{i=1}^m \frac{\nu_i}{|\omega_i||\bar{\omega}_i|} \sum_{p=1}^{|\omega_i|} \sum_{q=1}^{|\bar{\omega}_i|} Q(w_i, y_{\omega_{ip}}, y_{\bar{\omega}_{iq}}) \right].   
 \end{displaymath}
Therefore we have that $(m/2) \cdot (f(S) -  f(\hat{S}_{ab}))$ is equal to (the notation $\omega_{ij}$ denotes the $j$th element in $\omega_i$) 
\footnotesize
\begin{eqnarray*} 
 &=&   \mathbb{E}_{\nu} \left[\sup_{Q \in \mathcal{Q}}\sum_{i=1}^m \frac{\nu_i}{|\omega_i||\bar{\omega}_i|} \sum_{p=1}^{|\omega_i|} \sum_{q=1}^{|\bar{\omega}_i|} Q(w_i, y_{\omega_{ip}}, y_{\bar{\omega}_{iq}})  - \sup_{Q \in \mathcal{Q}} \sum_{i=1}^m \frac{\nu_i}{|\omega_i'||\bar{\omega}_i'|} \sum_{p=1}^{|\omega_i'|} \sum_{q=1}^{|\bar{\omega}_i'|} Q(w_i, y_{\omega_{ip}'}, y_{\bar{\omega}_{iq}'}) \right]  \\ 
 &\leq&  \mathbb{E}_{\nu} \left[\sup_{Q \in \mathcal{Q}} \left(\sum_{i=1}^m \frac{\nu_i}{|\omega_i||\bar{\omega}_i|} \sum_{p=1}^{|\omega_i|} \sum_{q=1}^{|\bar{\omega}_i|} Q(w_i, y_{\omega_{ip}}, y_{\bar{\omega}_{iq}})   -  \sum_{i=1}^m \frac{\nu_i}{|\omega_i'||\bar{\omega}_i'|} \sum_{p=1}^{|\omega_i'|} \sum_{q=1}^{|\bar{\omega}_i'|} Q(w_i, y_{\omega_{ip}'}, y_{\bar{\omega}_{iq}'})\right) \right] \\ 
  &=&  \mathbb{E}_{\nu} \left[\sup_{Q \in \mathcal{Q}} \left(\frac{\nu_a}{|\omega_a||\bar{\omega}_a|} \sum_{p=1}^{|\omega_a|} \sum_{q=1}^{|\bar{\omega}_a|} Q(w_a, y_{\omega_{ap}}, y_{\bar{\omega}_{aq}})   -   \frac{\nu_a}{|\omega_a'||\bar{\omega}_a'|} \sum_{p=1}^{|\omega_a'|} \sum_{q=1}^{|\bar{\omega}_a'|} Q(w_a, y_{\omega_{ap}'}, y_{\bar{\omega}_{aq}'})\right) \right] \\ 
    &=&  \mathbb{E}_{\nu} \left[\sup_{Q \in \mathcal{Q}} \left(\frac{\nu_a}{|\omega_a||\bar{\omega}_a|}  \left( \sum_{y_q \in \bar{\omega}_a \setminus \{y_{b'}\}} (Q(w_a, y_{b}, y_q) - Q(w_a, y_{b'}, y_q)) \right. \right. \right. \\  
   &&  \left. \left. \left. + \sum_{y_p \in \omega_a \setminus \{y_{b}\}} (Q(w_a, y_{p}, y_{b'}) - Q(w_a, y_p, y_b)) + Q(w_a, y_{b}, y_{b'}) - Q(w_a, y_{b'}, y_b) \right)   \right) \right] \\ 
&\leq& \frac{|\omega_a|+|\bar{\omega}_a|-1}{|\omega_a||\bar{\omega}_a|} \\ 
&=& \frac{n-1}{|\omega_a||\bar{\omega}_a|}.
\end{eqnarray*}
\normalsize
A similar derivation can be shown to bound $m(f(\hat{S}_{ab})- f(S))$ using an identical term. If we let $c_{ab} = \frac{2(n-1)}{m|\omega_a||\bar{\omega}_a|}$ then we can apply Theorem \ref{thm:mcdiarmid} and write $\mathbb{E}[f(S)] \leq f(S) + \epsilon$ with a probability greater than $1-\delta$ for some $\delta \in [0,1]$ where 
\begin{displaymath} 
\delta =  \exp\left(\frac{-2\epsilon^2}{\sum_{i=1}^m \sum_{j=1}^{n_i} c_{ij}^2}\right),
\end{displaymath}
and rearranging gives $\epsilon = \sqrt{\frac{2\ln(1/\delta) (n-1)^2}{m^2} \sum_{i=1}^{m} \frac{1}{|\omega_i||\bar{\omega}_i|^2}}$ as required.  
\end{proof}
To gain some insight into Theorem \ref{thm:rad} we study the cases for which the empirical Rademacher complexity is close to its expectation. In the first setting, $|\omega_i|$ is fixed and non-zero. Therefore the second term of the theorem is 
$
 \sqrt{\frac{2\ln(1/\delta) (n-1)^2}{m|\omega_i|(n-|\omega_i|)^2}},
$
which tends to zero as soon as the number of users $m$ tends to infinity. Similarly, if $|\omega_i|=\Theta(n)$ when $n$ tends to infinity, the second term of the theorem is $\Theta(1/\sqrt{mn})$ when $n$ and $m$ tends to infinity. Finally, when  $|\bar{\omega}_i|$ is fixed and non-zero, the second term grows as $\Theta(\sqrt{\frac{n}{m}})$ when $n$ and $m$ tend to infinity. In the last setting we require $m$ to grow faster than $n$ so that $n/m$ tends to zero to enforce the empirical Rademacher complexity to be close to its expectation.

We can now turn to concrete definitions of the function class $\mathcal{Q}$ relating to the matrix factorisation scenario we are interested in. Consider a specific form of the scoring function of the form $\mathcal{Q}_h = \{Q: (w_i, y_p, y_q) \mapsto h(\uv^T_i\vv_p - \uv^T_i\vv_{q}) \; | \; \|\Um\|_F \leq R, \|\Vm\|_F \leq R \}$ which allows us to analyse more deeply the empirical Rademacher complexity. Our goal is to bound the empirical Rademacher complexity based on the observations $S$. Before presenting a result we introduce two useful theorems.  

\begin{theorem}[\cite{meir2003generalization}] \label{thm:lipsch}
Let $\phi_1, \ldots, \phi_n$ be functions with respective Lipschitz constants $\gamma_1, \ldots, \gamma_n$, then the following bound holds: 
\begin{displaymath} 
\mathbb{E}_\sigma\left[ \sup_{f \in \mathcal{F}} \sum_{i=1}^n \sigma_i \phi_i(f(x_i)) \right]  \leq  \mathbb{E}_\sigma\left[ \sup_{f \in \mathcal{F}} \sum_{i=1}^n  \sigma_i \gamma_i f(x_i) \right],
\end{displaymath}
where $\mathcal{F}$ is a function class on $x \in \mathcal{X}$ and $\sigma_1, \ldots, \sigma_n$ are independent Rademacher variables. 

\end{theorem}

\begin{theorem}[\cite{eisenstat1995relative}]\label{thm:svdPert}
Let $\Bm \in \mathbb{R}^{m \times n}$ be a matrix with singular values $\sigma_1, \ldots, \sigma_r$ where $r$ is the rank of $\Am$, and $\tilde{\Bm} = \Dm_L\Bm\Dm_R$ be a multiplicative perturbation in which $\Dm_L$ and $\Dm_R$ are nonsingular matrices. The the following holds on the singular values $\tilde{\sigma}_1, \ldots, \tilde{\sigma}_r$ of $\tilde{\Bm}$: 
\begin{displaymath} 
 \frac{\sigma_i}{\|\Dm_L^{-1}\|_2\|\Dm_R^{-1}\|_2} \leq \tilde{\sigma_i} \leq \sigma_i \|\Dm_L\|_2\|\Dm_R\|_2, 
\end{displaymath}
where $\|\Am\|_2 = \sqrt{\lambda_{max}(\Am^T\Am)} = \sigma_{max}(\Am)$ is the spectral norm.  
\end{theorem}

We also introduce the following lemma whose utility will become apparently in the sequential theorem.

\begin{lemma}\label{lem:supTr}
Consider a matrix $\Am \in \mathbb{R}^{m \times n}$, then for a fixed integer $k$ the solutions $\Um \in \mathbb{R}^{m \times k}$ and $\Vm \in \mathbb{R}^{n \times k}$ of,
\begin{displaymath} 
\begin{array}{c c}
\sup & \tr(\Um^T \Am \Vm) \\  
\st & \|\Um\|_F = 1 \\ 
& \|\Vm\|_F = 1, 
\end{array}
\end{displaymath}
are given by $\Um = \frac{1}{\sqrt{k}}[\qv_1 \cdots  \qv_1]$ and $\Vm = \frac{1}{\sqrt{k}}[\pv_1 \cdots \pv_1]$ where $\pv_1$ and $\qv_1$ are the largest left and right singular vectors of $\Am$ respectively. The corresponding value of the objective is $\sigma_1$, the largest singular value of $\Am$. 
\end{lemma}
\begin{proof} 
We introduce $\phi(\Um, \Vm) =  \tr(\Um^T \Am \Vm) - \frac{1}{2}\lambda_1 (\tr(\Um^T\Um) - 1) - \frac{1}{2}\lambda_2 (\tr(\Vm^T\Vm) - 1)$ where $\lambda_1$ and $\lambda_2$ are Lagrange multipliers. Taking derivatives, one obtains the following equations: 
\begin{eqnarray} 
&&\Am \Vm  = \lambda_1 \Um \\ 
&&\Am^T \Um  = \lambda_2 \Vm, 
\end{eqnarray}
where $\lambda_1$ and $\lambda_2$ are Lagrange multipliers. By premultiplying the first equality by $\Um^T$ and the second by $\Vm$ then taking the trace one can see that $\lambda_1 = \lambda_2 = \lambda$ is the objective value. By substitution of $\Um = \lambda^{-1} \Am \Vm$ into the second equality we obtain $\Am^T\Am\Vm = \lambda^2 \Vm$ which implies that columns of $\Vm$ must be composed of a particular eigenvector of $\Am^T\Am$. In a similar manner, one obtains $\Am\Am^T\Um = \lambda^2 \Um$. Denote the left and right singular values of $\Am$ as $\pv_1, \ldots, \pv_r$ and $\qv_1, \ldots, \qv_r$, where $r$ is the rank of $\Am$, with corresponding singular values $\sigma_1 \geq \sigma_2 \geq \ldots \geq \sigma_r$, then it is clear that $\Um = a_1[\qv_1 \cdots  \qv_1]$, $\Vm = a_2[\pv_1 \cdots \pv_1]$ and $\lambda = \sigma_1$ for some constants $a_1$ and $a_2$. To find the scaling factors $a_1$ and $a_2$ we can write $\tr(\Um^T\Um) = a_1k^2 = 1$ and $\tr(\Vm^T\Vm) = a_2k^2 = 1$ which implies $a_1 = a_2 = 1/\sqrt{k}$.  
\end{proof}

These theorems allow us to make the following statement concerning the empirical Rademacher complexity under a matrix factorisation setting. 

\begin{theorem} 
Consider a sample $S = \cup_{i=1}^m \{(w_i, y_{\omega_{i1}}), \ldots, (w_i, y_{\omega_{in_i}})\}$ drawn from distribution $\mathcal{D}_1^{n_1} \times \cdots \times \mathcal{D}_m^{n_m}$. Define scoring functions of the form $\mathcal{Q}_h = \{Q: (w_i, y_p, y_q) \mapsto h(\uv^T_i\vv_p - \uv^T_i\vv_{q}) \; | \; \Um \in \mathbb{R}^{m \times k}, \Vm \in \mathbb{R}^{n \times k} \|\Um\|_F \leq R_\Um , \|\Vm\|_F \leq R_\Vm \}$ where $h$ is Lipschitz with constant $B$ and $k$ is a fixed integer. Then the following bound holds on the empirical Rademacher complexity
\begin{displaymath} 
  \hat{R}^{AUC}(\mathcal{Q}_h) \leq \frac{2B R_\Um R_\Vm}{m} \|\Em - \bar{\Em} \|_2, 
\end{displaymath}
where $\Sigmam$ is a diagonal matrix of the first $k$ singular values of $(\Em - \bar{\Em})$ with $\Em = (\diag(\Xm\jv)^{-1}\Xm)^T$, $\bar{\Em} = (\diag(\bar{\Xm}\jv)^{-1}\bar{\Xm})^T$ and $\bar{\Xm} = \Jm - \Xm$ where $\diag(\cdot)$ is a diagonal matrix of its vector input. 
\end{theorem}
\begin{proof}
Consider the following
\begin{eqnarray*} 
\hat{R}^{AUC}(\mathcal{Q}) &=& 2\mathbb{E}_{\nu} \left[\sup_{Q \in \mathcal{Q}} \frac{1}{m}\sum_{i=1}^m \frac{\nu_i}{n_i n_i'} \sum_{p=1}^{n_i} \sum_{q=1}^{n_i'} h(\uv_i^T\vv_p - \uv_i^T\vv_q) \right]   \\ 
  &\leq& \frac{2B}{m} \mathbb{E}_{\nu}\left[\sup_{\|\Um\|_F \leq R_\Um, \|\Vm\|_F \leq R_\Vm} \sum_{i=1}^m \frac{\nu_i}{n_i n_i'} \sum_{p \in \omega_i} \sum_{q \in \bar{\omega}_i} (\uv_i^T\vv_p - \uv_i^T\vv_q)  \right] \\ 
 &=& \frac{2B}{m} \mathbb{E}_{\nu}\left[\sup_{\|\Um\|_F \leq R_\Um, \|\Vm\|_F \leq R_\Vm} \sum_{i=1}^m \nu_i (\uv_i^T \dot{\vv}_i - \uv_i^T\ddot{\vv}_i)  \right]  \\ 
  &=& \frac{2B}{m} \mathbb{E}_{\nu}\left[\sup_{\|\Um\|_F \leq R_\Um, \|\Vm\|_F \leq R_\Vm} \tr(\Pim^{\nu}\Um\Vm^T\Em - \Pim^{\nu}\Um\Vm^T\bar{\Em})  \right]  \\ 
    &\leq& \frac{2BR_\Um R_\Vm}{m} \mathbb{E}_{\nu}\left[ \| (\Em - \bar{\Em})\Pim^{\nu} \|_2  \right]  \\ 
&\leq& \frac{2B R_\Um R_\Vm}{m} \|\Em - \bar{\Em} \|_2
\end{eqnarray*}
where $\Pim_{ii}^{\nu} = \nu_i$ for all $i$ and off-diagonal elements are zero. The second line results from an application of Theorem  \ref{thm:lipsch}. The fourth line is a matrix representation of the sum in the $\sup$ term and in the fifth line we use Lemma \ref{lem:supTr}. For the final line, note that the diagonal elements of $\Pim^{\nu}$ are selected from $\{-1, +1\}$ and so $\|\Pim^{\nu} \|_2 = 1$ and an application of Theorem \ref{thm:svdPert} shows the singular values do not change, giving the required result.  
\end{proof}

Notice that this bound does not depend on the dimensionality of the factor matrices $\Um$ and $\Vm$. Here we see the motivation for penalising Optimisation \ref{eqn:obj} by the norm of the weight matrices: doing so keeps the Rademacher complexity low. We would also benefit from choosing loss functions whose Lipschitz constant is small. Putting the parts together allows us to note the following. 

\begin{corollary}
Consider a sample $S = \cup_{i=1}^m \{(w_i, y_{\omega_{i1}}), \ldots, (w_i, y_{\omega_{in_i}})\}$ drawn from distribution $\mathcal{D}_1^{n_1} \times \cdots \times \mathcal{D}_m^{n_m}$. Using the notations defined above the following bound holds on the Rademacher complexity of $\mathcal{Q}_h$, with probability greater than $1-\delta$,
\begin{displaymath} 
 R^{AUC}(\mathcal{Q}_h) \leq \frac{2B R_\Um R_\Vm}{m} \|\Em - \bar{\Em} \|_2 +  \sqrt{\frac{2\ln(1/\delta)(n-1)^2}{m^2} \sum_{i=1}^{m} \frac{1}{|\omega_i||\bar{\omega}_i|^2}  }. 
\end{displaymath}
\end{corollary}

We now want to make the connection between the expectation of the AUC and the Rademacher complexity. To do so we introduce the following lemma. 

\begin{lemma}\label{lem:sup1} 
Let $\mathcal{Q}$ be a function class mapping from  $\mathcal{W}\times\mathcal{Y}^2$ to $[0, 1]$. Then with probability at least $1-\delta$ over all samples $S$ drawn from $\mathcal{D}$, the following holds: 
\begin{displaymath} 
 \sup_{Q \in \mathcal{Q}} (\mathbb{E}_\mathcal{D}[Q] - \hat{\mathbb{E}}_S[Q]) \leq  \mathbb{E}_{S \sim \mathcal{D}}[ \sup_{Q \in \mathcal{Q}}(\mathbb{E}_\mathcal{D}[Q] - \hat{\mathbb{E}}_S[Q])]  +   \sqrt{\frac{\ln(1/\delta)}{2m^2} \sum_{i=1}^{m} \frac{1}{|\omega_i|} \left(\frac{n-1}{|\bar{\omega}_i|}\right)^2 }.
\end{displaymath}
\end{lemma}
\begin{proof}
We can use a similar proof technique to that of Theorem \ref{thm:rad}  to write 
\begin{displaymath} 
| \hat{\mathbb{E}}_S[Q] - \hat{\mathbb{E}}_{\hat{S}_{ab}}[Q] | \leq \frac{n-1}{m|\omega_a||\bar{\omega}_a|},
\end{displaymath}
Noting that the left hand side can be written as $|(\mathbb{E}[Q] -  \hat{\mathbb{E}}_{\hat{S}_{ab}}[Q]) - (\mathbb{E}[Q] - \hat{\mathbb{E}}_S[Q]) |$ and taking the supremum over $Q$ allows us to bound $|\sup_{Q \in \mathcal{Q}} (\mathbb{E}[Q] - \hat{\mathbb{E}}_S[Q])  - \sup_{Q \in \mathcal{Q}} (\mathbb{E}[Q] - \hat{\mathbb{E}}_{\hat{S}_{ab}}[Q])|$: 
\begin{eqnarray*} 
 &\leq&  |\sup_{Q \in \mathcal{Q}} (\mathbb{E}[Q] - \hat{\mathbb{E}}_S[Q])  - (\mathbb{E}[Q] - \hat{\mathbb{E}}_{\hat{S}_{ab}}[Q])| \\ 
&\leq& \frac{n-1}{m|\omega_a||\bar{\omega}_a|},
\end{eqnarray*}
Define the following function 
\begin{displaymath} 
 f'(S) = \sup_{Q \in \mathcal{Q}} (\mathbb{E}[Q] - \hat{\mathbb{E}}_S[Q]), 
\end{displaymath}
then the above bound can be used in conjunction with McDiarmid's theorem and we can say 
\begin{displaymath} 
P_\mathcal{D}(f'(S) - \mathbb{E}[f'] > \epsilon ) \leq \exp\left(\frac{-2\epsilon^2}{\sum_{i=1}^m \sum_{j=1}^{n_i} c_{ij}^2}\right),  
\end{displaymath}
where $c_{ab} = \frac{n-1}{m|\omega_a||\bar{\omega}_a|}$. If we equate the right side of the above to $\delta$ we have $\epsilon = \sqrt{\frac{\ln(1/\delta)}{2m^2} \sum_{i=1}^{m} \frac{1}{|\omega_i|} \left(\frac{n-1}{|\bar{\omega}_i|}\right)^2 }$. 
\end{proof}

Finally we show how the empirical expectations of $Q \in \mathcal{Q}$ found on different samples are related to the Rademacher complexity in the following lemma.

\begin{lemma}\label{lem:sup2}
Let $S$ and $\tilde{S}$ be sampled from the distribution $\mathcal{D}$ and consider a sequence of Rademacher variables $\nu_1, \ldots, \nu_m$, then the following statement is true 
\begin{displaymath}
\mathbb{E}_{S, \tilde{S}} \sup_{Q \in \mathcal{Q}} [\hat{\mathbb{E}}_{\tilde{S}}[Q] - \hat{\mathbb{E}}_{S}[Q]] \leq R^{AUC}(\mathcal{Q}).
\end{displaymath}
\end{lemma}
\begin{proof} 
We start by showing the following is equivalent to $\sup_{Q \in \mathcal{Q}} [\hat{\mathbb{E}}_{\tilde{S}}[Q] -  \hat{\mathbb{E}}_{S}[Q]] $: 
\begin{eqnarray*} 
 \mathbb{E}_{S, \tilde{S} \nu} \sup_{Q \in \mathcal{Q}} \left[\frac{1}{m}\sum_{i=1}^m \frac{1}{n_i n_i'} \sum_{p=1}^{n_i} \sum_{q=1}^{n_i'} (\nu_i Q(w_i, y_{\omega_ip}, y_{\bar{\omega}_iq})  - \nu_i Q(w_i, y_{\omega_ip}, y_{\bar{\omega}_iq}) ) \right], 
\end{eqnarray*}
since when $\nu_i = -1$ for all $i$ the two expressions are clearly the same, and $\nu_i = 1$ swaps the observations for the $i$th user in $S$ with the corresponding ones in $\tilde{S}$. Since both $S$ and $\tilde{S}$ are sampled from the the same distribution the expectation of the supremum over these sets is the same. Note also that above term is upper bounded by the following: 
\begin{eqnarray*}   
 \sup_{Q \in \mathcal{Q}} [\hat{\mathbb{E}}_{\tilde{S}}[Q] -  \hat{\mathbb{E}}_{S}[Q]]  &\leq& \mathbb{E}_{S, \tilde{S} \nu} \sup_{Q \in \mathcal{Q}} \left[\frac{1}{m}\sum_{i=1}^m \frac{1}{n_i n_i'} \sum_{p=1}^{n_i} \sum_{q=1}^{n_i'} \nu_i Q(w_i, y_{\omega_ip}, y_{\bar{\omega}_iq} ) \right]  \\
 &&+  \mathbb{E}_{S, \tilde{S} \nu} \sup_{Q \in \mathcal{Q}} \left[-\frac{1}{m}\sum_{i=1}^m \frac{1}{n_i n_i'} \sum_{p=1}^{n_i} \sum_{q=1}^{n_i'} \nu_i Q(w_i, y_{\omega_ip}, y_{\bar{\omega}_iq}) \right]  \\
 &=& 2\mathbb{E}_{S, \tilde{S} \nu} \sup_{Q \in \mathcal{Q}} \left[\frac{1}{m}\sum_{i=1}^m \frac{1}{n_i n_i'} \sum_{p=1}^{n_i} \sum_{q=1}^{n_i'} \nu_i Q(w_i, y_{\omega_ip}, y_{\bar{\omega}_iq}) \right]  \\ 
 &=& R^{AUC}(\mathcal{Q}),
\end{eqnarray*}
where the second line comes from the symmetry of distribution of the Rademacher variables. 
\end{proof}
We can now bound the expectation of the loss using the Rademacher complexity. 
\begin{theorem} 
Let $S$ and $\tilde{S}$ be sampled from the distribution $\mathcal{D}$ and let $\mathcal{Q}$ be a function class mapping from  $\mathcal{W}\times\mathcal{Y}^2$ to $[0, 1]$. The with probability at least $1-\delta$, the following holds:  
\begin{displaymath} 
\mathbb{E}_\mathcal{D}[Q] \leq  \hat{\mathbb{E}}_S[Q] + R^{AUC}(\mathcal{Q})  +   \sqrt{\frac{\ln(1/\delta)(n-1)^2}{2m^2} \sum_{i=1}^{m} \frac{1}{|\omega_i||\bar{\omega}_i|^2}  }.
\end{displaymath}
\end{theorem}
\begin{proof} 
The result follows directly from Lemmas \ref{lem:sup1} and \ref{lem:sup2}. 
\end{proof}

Given this result we can examine our loss functions in terms of the bound on the expectation. A key advantage of the bound is it is data-dependent and hence we can estimate the model complexity if we replace the Rademacher complexity term with the bound on its empirical estimate. We have already discussed when this quantity is small as well as a similar analysis for the final term. In practice one finds that the bound on the Rademacher term is larger than the last term hence once must be focus on this term. When studying the loss functions of Section \ref{sec:general} one can see that the logistic and sigmoid functions are Lipschitz with constant 1. If we say that the norms of $\uv_i$ and $\vv_i$ are bounded by $D$ then we have $-2D^2 \leq \gamma_{i,p,q} \leq 2D^2$. The hinge loss is given by $\frac{1}{2}\max(0, 1-x)^2$ which implies $D = 1/\sqrt{4}$ so that the loss term is in the range $[0, 1]$ and the Lipschitz constant is $3/2$.

\section{Related Work}

In this section we will briefly review some works on finding low rank matrix factorisations for implicit rating matrices. An early attempt to deal with implicit ratings using a squared loss is presented in \cite{hu2008collaborative,pan2008one}. The formulation is an adaptation of the Singular Value Decomposition (SVD), and minimizes
\begin{displaymath} 
\min \sum_{i=1}^m\sum_{j=1}^n \Cm_{ij} (\uv_i^T\vv_j - 1)^2 + \lambda(\|\Um\|_F^2 + \|\Vm\|_F^2),  
\end{displaymath}
where $\Cm$ is a matrix of weights for user-item pairs and $\lambda$ is a regularisation parameter. In \cite{pan2008one} the user-item values are set to 1 for positive items and lower constants for the rest. A related problem is that of \emph{matrix completion} \cite{candes2009exact,mazumder2010spectral} in which one minimises the Frobenius norm difference between real and predicted ratings using a trace norm penalisation (the sum of the singular values, denoted $\|\cdot\|_*$),
\begin{displaymath} \label{eqn:opt2}
 \min \frac{1}{2} \sum_{i,j \in \omega} (\Xm_{ij} - \Zm_{ij})^2  + \lambda \|\Zm\|_*,
\end{displaymath}
\noindent 
where $\lambda$ is a user-defined regularisation parameter and $\Zm$ is the matrix factorisation. Notice that only nonzero elements of $\Xm$ are considered in the cost. A key strength of matrix completion is the strong theoretical guarantees on finding unknown entries with high accuracy.  The disadvantage of both these approaches is that they do not take into account the orderings of the items. 

In \cite{rendle2009bpr} the authors study AUC maximisation for matrix factorisation in the context of recommendation, the primary motivation for which is a maximum posterior estimate for a Bayesian framing of the problem. The connection to AUC maximisation is made by replacing the indicator function in its computation with the log sigmoid denoted by $\ln \sigma(x) = \ln(1/(1+e^{-x}))$. Solutions are obtained using a stochastic gradient descent algorithm based on bootstrap sampling. The particular optimisation considered takes the form
\begin{equation*} 
\begin{split}
\max & \sum_{i=1}^m\sum_{p \in \omega_i}\sum_{q \in \bar{\omega}_i} \log \sigma(\uv_i^T\vv_p - \uv_i^T\vv_q) - \frac{\lambda_\Um }{2}\|\Um\|_F^2\\
&-  \sum_i \left(\frac{\lambda_{\Vm_1}}{2}\sum_{p \in \omega_i} \Vm_{ip}^2 + \frac{\lambda_{\Vm_0}}{2}\sum_{q \in \bar{\omega}_i} \Vm_{iq}^2\right),
\end{split}
\end{equation*}
where $\lambda_{\Um}, \lambda_{\Vm_1}, \lambda_{\Vm_0}$ are regularisation constants for the user factors, positive items and negative items respectively. The first term is a log-sigmoid unnormalised relaxation of the AUC criterion and the remaining terms are used for regularisation. Note that one optimises over the full list as opposed to prioritising the top few. Furthermore, our framework specialises to BPR in logistical loss case and when the regularisation parameters above are equivalent. 

Another paper which considers the AUC is \cite{weston2013learning} however it departs from other papers in using an item factor model of the form $\Zm_{ij} = \frac{1}{|\omega_i|} \sum_{p \in \omega_i} \vv_p^T\vv_j$ which is the score for the $i$th user and $j$th item. The disadvantage of the item modelling approach is that it does not model users separately. Indeed, the connection between the user-item factor approach can be seen by setting $\uv_i = \frac{1}{|\omega_i|} \sum_{p \in \omega_i} \vv_p$. In the optimisation, the AUC is approximated by the hinge loss,
\begin{displaymath} 
\min \sum_{i=1}^m\sum_{p\in \omega_i}\sum_{q\in \bar{\omega}_i} \max(0, 1 - \Zm_{ip} + \Zm_{iq}),
\end{displaymath}
and one applies stochastic gradient descent using one user, one positive and one negative item chosen at random at each gradient step. As noted by the authors, this loss does not prioritise the top of the list and hence they propose another loss  
\begin{displaymath} 
\min \sum_{i=1}^m\sum_{p\in \omega_i}\sum_{q\in \bar{\omega}_i} \Phi\left(\frac{|\bar{\omega}_i|}{N}\right)\max(0, 1 - \Zm_{ip} + \Zm_{iq}),  
\end{displaymath}
where $\Phi(x) = \sum_{i=1}^x 1/x$ and $N$ is a sampled approximation of $\sum_{q \in \bar{\omega_i}} \mathcal{I}(\Zm_{ip} \leq 1 - \Zm_{iq})$ i.e. the rank of $p$th item for $i$th user. Additionally, the algorithm samples ranks of positive items in non-uniform ways in order to prioritise the top of the list. 

A related approach based on Mean Reciprocal Rank (MRR) \cite{shi2012climf} increases the importance of top $k$-ranked items in conjunction with maximising a lower bound on the smoothed MRR of the list. In words, the MRR is the average of the reciprocal of the rank of the first correct prediction for each user. The authors use a sigmoid function to approximate the indicator and after relaxation of a lower bound of the reciprocal rank, the following function is maximised 
\begin{equation*}
% \begin{split}
% \max
\sum_{ij} \Xm_{ij} \left(\ln\sigma(\Zm_{ij}) +\sum_{k=1}^n\ln(1- \Xm_{ik} \sigma(\Zm_{ik} - \Zm_{ij})) \right) - \frac{\lambda}{2}(\|\Um\|^2_F + \|\Vm\|_F^2),
% \end{split}
\end{equation*}
where $\lambda$ is a regularisation constant. The first term in the sum promotes elements in the factor model that correspond to relevant items and the second term degrades the relevance scores of the irrelevant items relative to relevant item $y_j$. 

\section{Computational Experiments}\label{sec:exp} 

In this section we analyse various properties of MFAUC empirically in order to understand its behaviour and ranking performance. We consider the question of how well MFAUC optimises the AUC on the training set under specific loss functions and weightings. Another important question is whether the parallel optimisation procedure can speed up convergence of the objective. A final consideration is the evaluation of our ranking framework and other matrix factorisation methods when considering results at the very top of the list for each user. For this comparison we benchmark against other user-item matrix factorisation methods including Soft Impute \cite{mazumder2010spectral} and Weighted Regularised Matrix Factorisation (WRMF, \cite{hu2008collaborative}). Notice that the choice of the logistic loss function and identity weighting implies a comparison to BPR. All experimental code is written in Python with critical sections written in Cython and C++.  

We make use of the following synthetic datasets. The first, \texttt{Synthetic1}, is generated in the following manner. Two random matrices $\Um^* \in \mathbb{R}^{500 \times 8}$ and $\Vm^* \in \mathbb{R}^{200 \times 8}$ are constructed such that respectively their columns are orthogonal. We then compute the partially observed matrix $\hat{\Xm}_{ij} = \mathcal{I}(\uv_i^T\vv_j > Q(s, 1-t))$ where $Q(s, 1-t_i)$ represents the quantile corresponding to the top $t_i$ items. Here $t_i = .1$ so that there are on average 20 nonzero elements per row of $\hat{\Xm}$, and we additionally add an average of 5 random relevant ratings per row to form our final rating matrix $\Xm$. For the second dataset, \texttt{Synthetic2}, we generate $\Um$ and $\Vm$ in a similar way, however this time the probability of observing a rating (relevant/irrelevant) is distributed according to the power law for each item/user with exponent 1. We iteratively sample observations according to this distribution of users and items, setting $\Xm_{ij}$ to be 1 if $\Zm_{ij} \geq \hat{\mathbb{E}}[\Zm]$, $\Zm = \Um\Vm^T$, otherwise it is zero. We continue this process until the density of the matrix is at least $.1$. The final matrix is of size $573 \times 300$ with $17796$ non-zeros. 

\subsection{ROC Analysis of Losses}

First we analyse the maximisation of AUC under the loss functions we outlined in Section \ref{sec:general} using the synthetic datasets.  The MFAUC algorithm is set up with $k=8$ using $\kappa_{\mathcal{W}} = 30$ row samples and $\kappa_{\mathcal{Y}} = 10$ column samples to compute approximate derivatives for $\uv_i$ and $\vv_i$. The initial values of $\Um$ and $\Vm$ are set to zero mean Gaussian random values with a standard deviation of .1, the regularisation constant $\lambda = 0$, maximum iterations $T=500$ and item distribution exponent $\tau = 0$. The learning rate is  $\alpha = .05$ and we choose $\beta \in \{.5, 1.0, 2.0\}$ for sigmoid and logistic losses and $\rho \in \{.5, 1.0, 2.0\}$ for the $\tanh$ item weighting. To make the problem harder we remove 5 items for each user and then train using the remaining items, recording the ROC curve at the end of the procedure. This process is repeated 5 times, averaging results. Since we are interested in the top items of the list, we consider items until a false positive rate of 20\% is encountered. For clarity we only include the best ROC curves for logistic, sigmoid and tanh-based losses, defined as the largest true positive rate at 20\% false positives.  

\begin{figure}[ht]
\centering

\subfigure[\texttt{Synthetic1}]{
   \includegraphics[width=.45\linewidth] {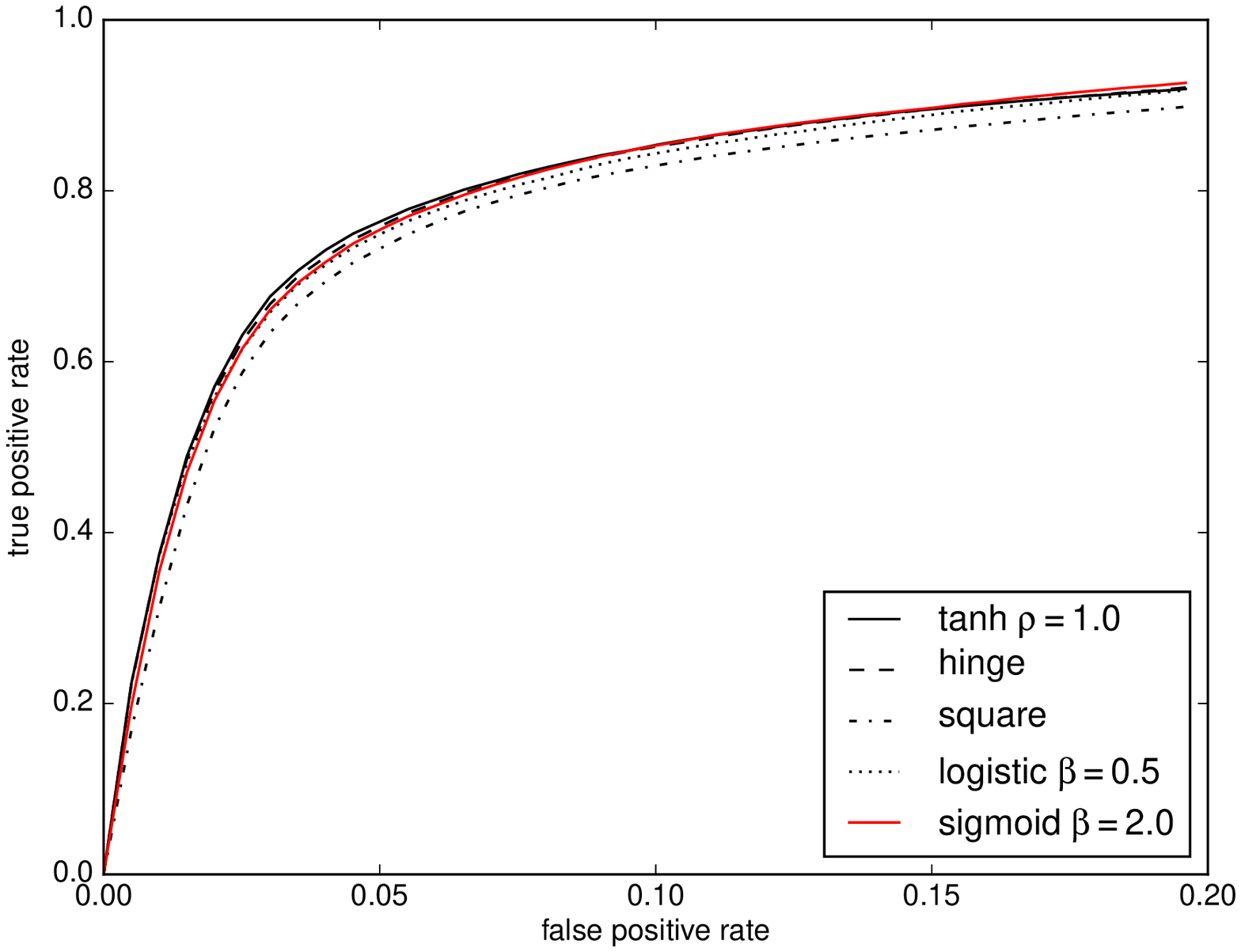}
 }
 \subfigure[\texttt{Synthetic2}]{
   \includegraphics[width=.45\linewidth] {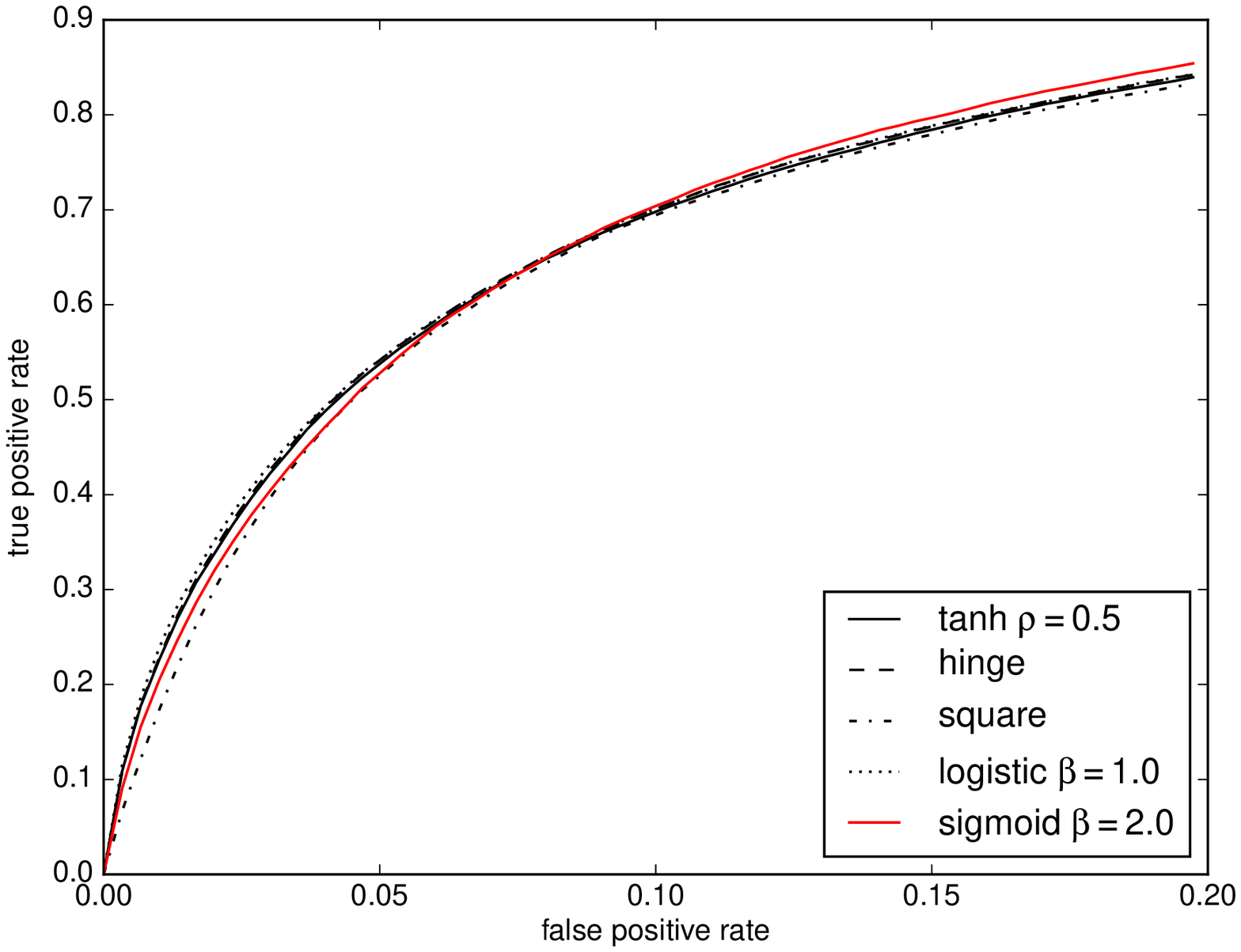}
 }

\caption{Left side of ROC curves for the synthetic datasets using different loss functions.}\label{fig:roc}
\end{figure}

The ROC plots are presented in Figure \ref{fig:roc} and on the whole we see that differences are slight for both datasets. On \texttt{Synthetic1} the $\tanh$ weighting function is slightly preferential at the start of the curve relative to the other losses, followed by hinge, sigmoid and then the logistic loss. The square loss provides the worst of the results as it penalises only the difference in the scores for positive and negative items, however one does gain a speed advantage as outlined in Section \ref{subsec:opt}. We can see from the curves on \texttt{Synthetic2} that this is a more challenging problem. We observe the logistic loss providing the fewest errors at the start of the ranking followed closely by hinge and then the $\tanh$ loss. The sigmoid and square loss curves are poor in this case, the latter for the reason we have already mentioned. Performing gradient descent over the sigmoid loss function can be challenging for the reasons we mentioned in Section \ref{sec:general}.    

\subsection{Optimisation Strategy} 

We now examine the parallel optimisation proposed in Section \ref{subsec:opt} in terms of convergence and timing. The MFAUC algorithm is set up with $k=8$ using $\kappa_{\mathcal{U}} = 30$ row samples and $\kappa_{\mathcal{Y}} = 10$ column samples to compute approximate derivatives for $\uv_i$ and $\vv_i$. The initial values of $\Um$ and $\Vm$ are set to random Gaussian values, then we fix learning rate $\alpha = .05$, regularisation $\lambda = .1$, maximum iterations $T=300$, and use the hinge loss. To see how the parallelisation affects convergence and timings, we record these values for the parallel optimisation with 2, 4, and 8 processes, as well as the standard SGD approach, repeating experiments 5 times to get average quantities for each parameter set. The experiment is run on an Intel core i7-3740QM CPU with 8 cores and 16 GB of RAM and we record the objective value every 5 iterations.  

\begin{figure}[ht]
\centering%
\subfigure[\texttt{Synthetic1}]{
   \includegraphics[width=.45\linewidth] {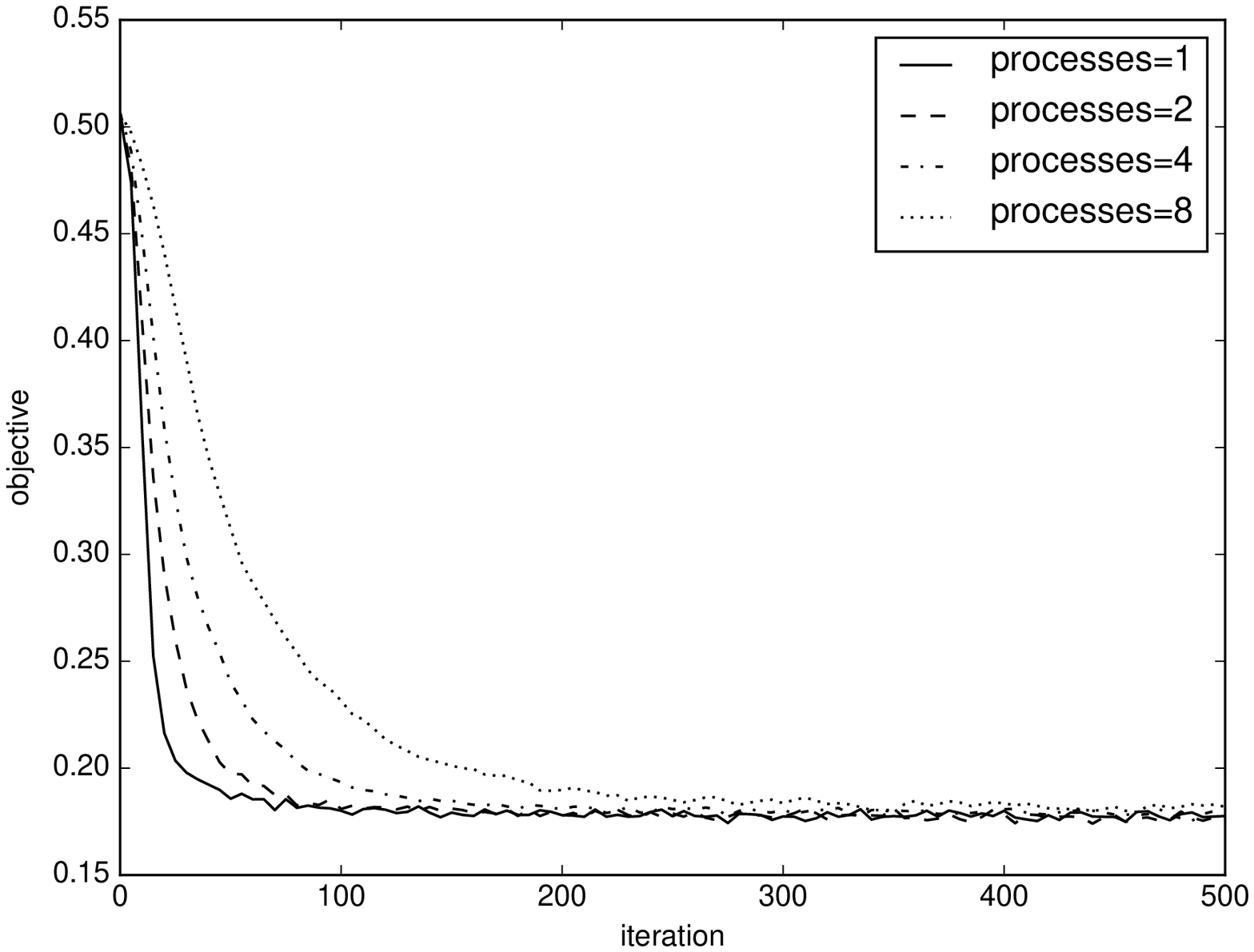}
 }
 \subfigure[\texttt{Synthetic2}]{
   \includegraphics[width=.45\linewidth] {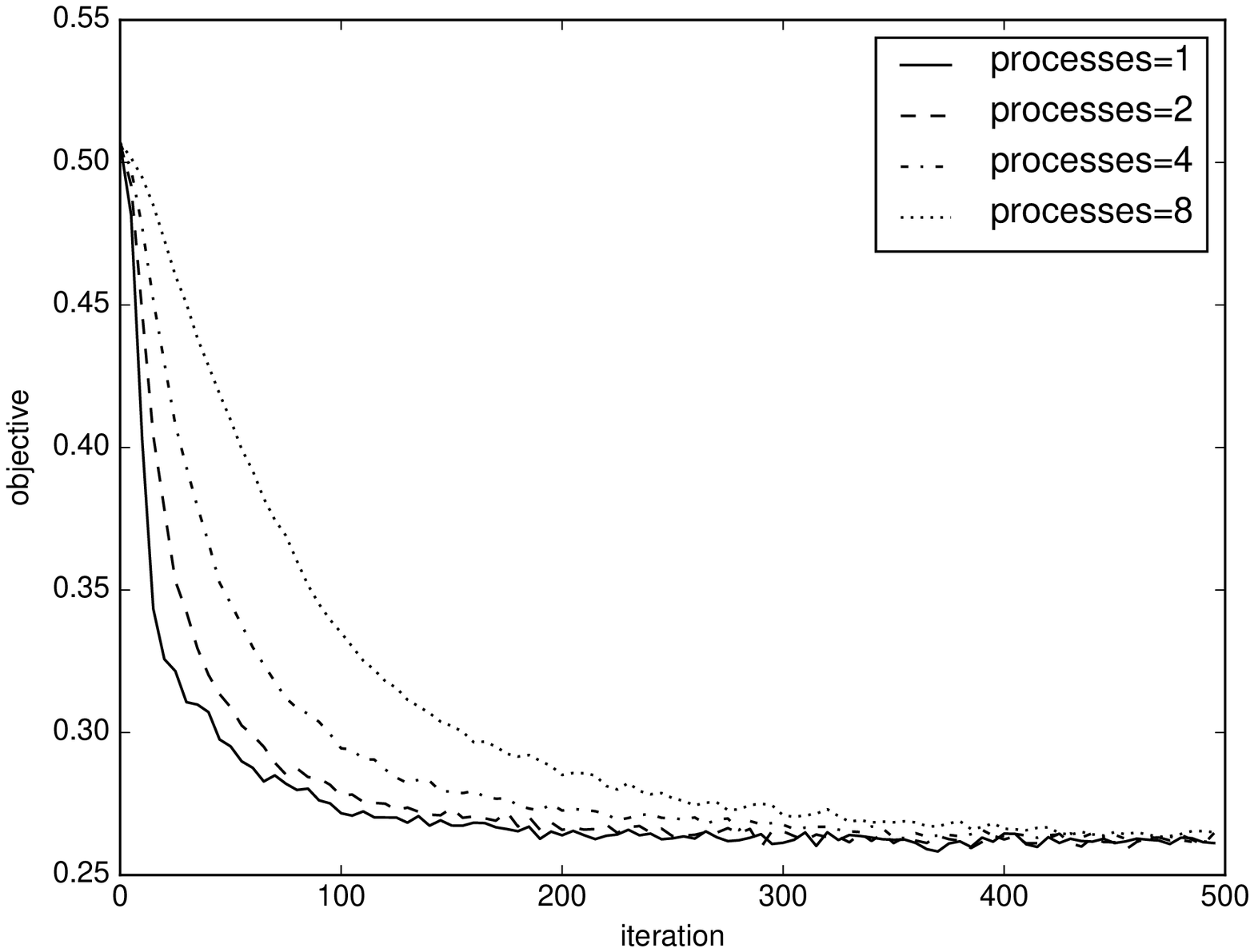}
 }

\caption{Plots showing the objective function on the synthetic datasets with different optimisation routines.}\label{fig:timings} 
\end{figure}

\begin{table}
\centering
\begin{tabular}{|l |l l l l|} 
\hline
 & 1 & 2 & 4 & 8 \\ 
\hline
\texttt{Synthetic1} & 607.1 & 309.1 & 163.6 & 106.9 \\ 
\texttt{Synthetic2} & 689.2 & 332.8 & 201.9 & 16.3 \\ 
\hline
\end{tabular}
\caption{Timings (in seconds) of the optimisation routines with the synthetic datasets by number of processes.}\label{tab:timings}
\end{table}

Figure \ref{fig:timings} compares the objective values and Table \ref{tab:timings} shows the timings of the parallel and non-parallel variants of SGD on both datasets.  We observe approximate speedups of 3 times with 8 processes whilst converging to approximately the same objective values for both datasets. When we look at the objective against the iteration number, the parallel SGD converges slower than the non-parallel version particularly when using 8 cores. One of the reasons for this decrease in convergence rate is that the dataset is split into smaller blocks when more processes are applied. Overall however, parallel SGD matches the objective of SGD at a few smaller time cost and we naturally expect this improvement factor to increase with higher core CPUs. 

\subsection{Comparative Ranking Performance} 

Here we concern ourselves with how well the different matrix factorisation methods can rank items in a top-$\ell$ recommendation task. From each user, 5 randomly selected relevant elements are removed and then a prediction is made for the top $\ell$ items. In this case,  $\ell \in \{1, 3, 5\}$ and the average values of the precision, recall and AUC are recorded. 

As an initial step in the learning process we perform model selection on the training nonzero elements using 3-fold cross validation. For MFAUC the parameters are identical to the setup used above and we select learning rates from $\alpha \in \{2^{-1}, \ldots, 2^{-8}\}$ and $\lambda \in \{2^{0}, \ldots, 2^{-10}\}$. The initial factor matrices are computed using the randomised SVD. The F1 measure is recorded on validation items in conjunction with MFAUC and used to pick the best solution. We take a sample of 3 validation items from the users to form this validation set.  With Soft Impute, we use regularisation parameters $\lambda \in \{1.0, .8,\ldots, 0\}$ and the randomised SVD to update solutions as proposed in \cite{dhanjal14online}. The regularisation parameters for WRMF are chosen from $\lambda \in \{2^{1}, 2^{-1}, \ldots, 2^{-11}\}$.  To select the best parameters we use the maximum F1 scores on the test items averaged over all folds, fixing $k=8$ as this is the dimension used to generate the data. Once we have found the optimal parameters we train using the training observations and test on the remaining elements to get estimates of precision, recall and AUC. The training is repeated 5 times with different random seeds and the resulting evaluation metrics are averaged. 

\begin{table*} 
\centering
\begin{tabular}{|l|l |l l l | l l l | l |} 
\cline{3-9}
\multicolumn{1}{c}{}&& p@1 & p@3 & p@5 & r@1 & r@3 & r@5 &  AUC\\
\hline
\multirow{10}{*}{\rotatebox[origin=c]{90}{\texttt{Syn1}}}
 & SoftImpute		 & .831 & .696 & .549 & .166 & .417 & .549 & .914\\
 & WRMF		 & \textbf{.893} & .754 & .610 & \textbf{.179} & .452 & .610 & \textbf{.924}\\
 & MFAUC hinge          	 & .855 & .759 & .606 & .171 & .455 & .606 & \textbf{.924}\\
 & MFAUC square         	 & .863 & .755 & .607 & .173 & .453 & .607 & .922\\
 & MFAUC sigmoid        	 & .867 & .772 & \textbf{.626} & .173 & .463 & \textbf{.626} & .923\\
 & MFAUC logistic       	 & .872 & .762 & .614 & .174 & .457 & .614 & \textbf{.924}\\
 & MFAUC tanh $\rho=.5$	 & .854 & .752 & .602 & .171 & .451 & .602 & .923\\
 & MFAUC tanh $\rho=1.0$	 & .866 & .760 & .615 & .173 & .456 & .615 & \textbf{.924}\\
 & MFAUC tanh $\rho=2.0$	 & .874 & \textbf{.775} & .618 & .175 & \textbf{.465} & .618 & .921\\
 & MFAUC tanh $\rho=5.0$	 & .880 & .764 & .604 & .176 & .459 & .604 & .920\\
\hline
\multirow{10}{*}{\rotatebox[origin=c]{90}{\texttt{Syn2}}}
 & SoftImpute		 & .225 & .193 & .172 & .045 & .116 & .172 & .766\\
 & WRMF		 & .431 & \textbf{.297} & \textbf{.241} & .086 & \textbf{.178} & \textbf{.241} & \textbf{.817}\\
 & MFAUC hinge          	 & .428 & .294 & .240 & .086 & .176 & .240 & .796\\
 & MFAUC square         	 & .422 & .290 & .236 & .084 & .174 & .236 & .801\\
 & MFAUC sigmoid        	 & .416 & .284 & .228 & .083 & .170 & .228 & .795\\
 & MFAUC logistic       	 & \textbf{.444} & .295 & .238 & \textbf{.089} & .177 & .238 & .800\\
 & MFAUC tanh $\rho=.5$	 & .415 & .284 & .232 & .083 & .171 & .232 & .791\\
 & MFAUC tanh $\rho=1.0$	 & .420 & .284 & .230 & .084 & .170 & .230 & .790\\
 & MFAUC tanh $\rho=2.0$	 & .395 & .285 & .231 & .079 & .171 & .231 & .797\\
 & MFAUC tanh $\rho=5.0$	 & .362 & .259 & .214 & .072 & .156 & .214 & .776\\
\hline
\end{tabular}
\caption{Test errors on the synthetic datasets. Top represents \texttt{Synthetic1} and bottom is \texttt{Synthetic2}\label{tab:syntheticRes1}. Best results are in bold. }
\end{table*}

Table \ref{tab:syntheticRes1} shows the performance of the matrix factorisation methods. On both datasets, MFAUC and WRMF perform better than Soft Impute and particularly on \texttt{Synthetic2}. One explanation is that WRMF and MLAUC do not make assumptions about the distributions of the relevant items like Soft Impute. On the harder \texttt{Synthetic2} dataset WRMF gives the best overall results closely followed by logistic and hinge loss MLAUC. 

\subsubsection{Real Datasets}

Next we consider a set of real world datasets: MovieLens, Flixster, Epinions and Book Crossing. For the MovieLens and Flixster datasets, ratings are given on scales of 1 to 5 and those greater than 3 are considered relevant with the remaining ones set to zero. Epinions ratings are given on the scale 0 to 5 and Book Crossing ones are given from 0 to 10, and we assign relevance to ratings greater than 3 and 4 respectively. For all datasets, we remove users with less than 10 items and items with less than 2 users, repeating this process until convergence is reached. Properties about the resulting matrices are shown in Table \ref{tab:datasets}. 

\begin{table}
\centering
\begin{tabular}{|l |c c c c|} 
\hline
Dataset & users & items & nonzeros & sparsity (\%) \\ 
\hline
Book Crossing & 9571 & 68,517 & 640,430 & 0.098 \\ 
Epinions & 12,663 & 38,499 & 371,969 & 0.076 \\
Flixster & 43,979 & 32,024 & 5,147,187 & 0.37 \\ 
MovieLens & 897 & 1281 & 54,883 & 4.78 \\ 
\hline
\end{tabular}
\caption{Properties of the real datasets}\label{tab:datasets}
\end{table}

The experimental procedure is similar to before except that we select $k \in \{32, 64, 128\}$ in this case and use 2 model selection repetitions. Since the datasets are larger than the synthetic ones we perform model selection on a subsample of at most $10^5$ ratings from the complete matrices. To form the model selection matrix, we pick users sequentially until the desired number of ratings is reached. Any items which then have no ratings are removed.  The parameters for MFAUC are chosen from $\alpha \in \{2^{1}, \ldots, 2^{-2}\}$, $\lambda \in \{2^{0}, \ldots, 2^{-5}\} $ and $\kappa_{\mathcal{W}} = 15$. After the model selection step, we use the parallel SGD procedure in conjunction with MLAUC to compute the final matrix factorisation.

\begin{table*}
\centering
\begin{tabular}{|l|l |l l l | l l l | l |} 
\cline{3-9}
\multicolumn{1}{c}{}&& p@1 & p@3 & p@5 & r@1 & r@3 & r@5 &  AUC\\
\hline
\multirow{9}{*}{\rotatebox[origin=c]{90}{\texttt{Book Crossing}}}
 & SoftImpute		 & \textbf{.040} & \textbf{.031} & \textbf{.027} & \textbf{.008} & \textbf{.019} & \textbf{.027} & .782\\
 & WRMF		 & \textbf{.040} & .030 & .026 & \textbf{.008} & .018 & .026 & .754\\
 & MFAUC hinge          	 & .018 & .016 & .014 & .004 & .010 & .014 & .845\\
 & MFAUC sigmoid        	 & .015 & .013 & .012 & .003 & .008 & .012 & \textbf{.849}\\
 & MFAUC logistic       	 & .020 & .016 & .014 & .004 & .010 & .014 & .836\\
 & MFAUC tanh $\rho=.5$	 & .015 & .012 & .010 & .003 & .007 & .010 & .797\\
 & MFAUC tanh $\rho=1.0$	 & .018 & .015 & .013 & .004 & .009 & .013 & .833\\
 & MFAUC tanh $\rho=2.0$	 & .018 & .016 & .014 & .004 & .009 & .014 & .846\\
\hline
\multirow{9}{*}{\rotatebox[origin=c]{90}{\texttt{Epinions}}}
 & SoftImpute		 & .037 & .030 & .026 & .007 & .018 & .026 & .793\\
 & WRMF		 & \textbf{.041} & \textbf{.031} & \textbf{.027} & \textbf{.008} & \textbf{.019} & \textbf{.027} & .771\\
 & MFAUC hinge          	 & .025 & .020 & .019 & .005 & .012 & .019 & .826\\
 & MFAUC sigmoid        	 & .025 & .020 & .018 & .005 & .012 & .018 & .819\\
 & MFAUC logistic       	 & .034 & .027 & .023 & .007 & .016 & .023 & .853\\
 & MFAUC tanh $\rho=.5$	 & .027 & .024 & .021 & .005 & .014 & .021 & .824\\
 & MFAUC tanh $\rho=1.0$	 & .031 & .026 & .023 & .006 & .016 & .023 & .854\\
 & MFAUC tanh $\rho=2.0$	 & .033 & .028 & .024 & .007 & .017 & .024 & \textbf{.859}\\
\hline
\multirow{9}{*}{\rotatebox[origin=c]{90}{\texttt{Flixster}}}
 & SoftImpute		 & .157 & .111 & .089 & .031 & .067 & .089 & .926\\
 & WRMF		 & .167 & .119 & .096 & .033 & .071 & .096 & .891\\
 & MFAUC hinge          	 & \textbf{.168} & \textbf{.132} & \textbf{.112} & \textbf{.034} & \textbf{.079} & \textbf{.112} & \textbf{.984}\\
 & MFAUC sigmoid        	 & .121 & .090 & .076 & .024 & .054 & .076 & .980\\
 & MFAUC logistic       	 & .167 & .126 & .107 & .033 & .076 & .107 & \textbf{.984}\\
 & MFAUC tanh $\rho=.5$	 & .119 & .090 & .077 & .024 & .054 & .077 & .981\\
 & MFAUC tanh $\rho=1.0$	 & .058 & .049 & .043 & .012 & .029 & .043 & .967\\
 & MFAUC tanh $\rho=2.0$	 & .069 & .052 & .047 & .014 & .031 & .047 & .969\\
\hline
\multirow{9}{*}{\rotatebox[origin=c]{90}{\texttt{MovieLens}}}
 & SoftImpute		 & .209 & .165 & .140 & .042 & .099 & .140 & .880\\
 & WRMF		 & .225 & .174 & .145 & .045 & .104 & .145 & .891\\
 & MFAUC hinge          	 & .217 & .183 & .158 & .043 & .110 & .158 & .919\\
 & MFAUC square         	 & .211 & .176 & .157 & .042 & .106 & .157 & .925\\
 & MFAUC sigmoid        	 & .237 & \textbf{.193} & .165 & .047 & \textbf{.116} & .165 & .924\\
 & MFAUC logistic       	 & \textbf{.241} & .191 & \textbf{.166} & \textbf{.048} & .114 & \textbf{.166} & \textbf{.926}\\
 & MFAUC tanh $\rho=.5$	 & .223 & .185 & .161 & .045 & .111 & .161 & .922\\
 & MFAUC tanh $\rho=1.0$	 & .226 & .183 & .159 & .045 & .110 & .159 & .923\\
 & MFAUC tanh $\rho=2.0$	 & .222 & .180 & .157 & .044 & .108 & .157 & .918\\
\hline
\end{tabular}
\caption{Test errors on the real datasets. }\label{tab:realRes}
\end{table*}

Table \ref{tab:realRes} shows the results on these datasets. It is clear that the non-ranking based methods perform reasonably well on the more sparse datasets Epinions and Book Crossing relative to MLAUC. We noticed in the model selection stage for example, MLAUC would not converge adequately for many parameter sets. Whilst this did not negatively impact the AUC, it did effect the items at the very top of the list, hence the low precision and recall scores for these datasets. In contrast, we see with Flixster and MovieLens that the ranking methods show their advantage particularly with the hinge and logistic losses. With MovieLens for example, all of the ranking losses improve over WRMF in the precisions at 3 and 5. 

A useful comparison is between the hinge and $tanh$ losses since it demonstrates the effectiveness of prioritising list elements. The results are mixed: on the Epinions and MovieLens datasets we gain an improvement with this prioritisation function, but on Flixster results are worse. A difficultly of the approach is that one must correctly set $\rho$ for accurate results. 

\section{Conclusion}

Recommendation is a Learning To Rank (LTR) problem, in that for each user the system has to rank items according to the relevance for the user. Whilst early recommender systems based on matrix factorisation aimed to recover the complete matrix of ratings, recent advances focus on optimising scoring losses designed for LTR problems. The current paper pushes forward this domain by considering local AUC maximisation, which focuses on the ranking of top items. The corresponding loss is handled with a smooth surrogate function which is minimised through stochastic gradient descent. Furthermore, use of parallel architectures by a blockwise partitioning of the rating matrix in conjunction with stochastic gradient descent allows the algorithm to run on datasets with millions of known entries. In addition we show that our chosen loss functions are consistent with AUC and gained insight into the generalisation of the algorithms using Rademacher Theory. 

From the computational study we can conclude that the proposed parallelisation by block optimisation speeds up the convergence whilst keeping the quality of the objective relative to single process optimisation. The weighting of observations, which gives more importance to items which are more often relevant, can be effective for certain datasets. In general, the MFAUC approach is a useful tool when the sparsity of the underlying dataset under examination is not too high.  

\section*{Acknowledgements}

This work is funded by the Eurostars ERASM project.

\bibliographystyle{plain}
\bibliography{references}
%\bibliography{../../include/references}

\end{document}

%% file: CsdMacros.tex
\newcommand{\jv}{\textbf{j}}

\newcommand{\pv}{\textbf{p}}
\newcommand{\qv}{\textbf{q}}

\newcommand{\uv}{\textbf{u}}
\newcommand{\vv}{\textbf{v}}
\newcommand{\wv}{\textbf{w}}
\newcommand{\xv}{\textbf{x}}

\newcommand{\zv}{\textbf{z}}

\newcommand{\Sigmam}{\mbox{\boldmath$\Sigma$}}

\newcommand{\Pim}{\mbox{\boldmath$\Pi$}}

\newcommand{\diag}{\mbox{diag}}
\newcommand{\tr}{\mbox{tr}}

\newcommand{\st}{\mbox{s.t.}}

\newcommand{\Am}{\textbf{A}}
\newcommand{\Bm}{\textbf{B}}
\newcommand{\Cm}{\textbf{C}}
\newcommand{\Dm}{\textbf{D}}
\newcommand{\Em}{\textbf{E}}

\newcommand{\Jm}{\textbf{J}}

\newcommand{\Um}{\textbf{U}}
\newcommand{\Vm}{\textbf{V}}

\newcommand{\Xm}{\textbf{X}}

\newcommand{\Zm}{\textbf{Z}}